\documentclass[lettersize,journal]{IEEEtran}
\usepackage{amsmath,amsfonts,amsthm}
\usepackage{algorithmic}
\usepackage{algorithm}
\usepackage{array}
\usepackage{textcomp}
\usepackage{stfloats}
\usepackage{url}
\usepackage{hyperref}
\hypersetup{colorlinks=true}
\usepackage{verbatim}
\usepackage{graphicx, graphics, epsfig, color}
\usepackage{cite}
\usepackage{bbding} 
\usepackage{dsfont} 
\usepackage{tikz, pgfplots}
\usetikzlibrary{plotmarks,spy,backgrounds}
\pgfplotsset{compat=newest}
\usepackage{cleveref}
\usepackage{makecell}
\usepackage{booktabs}
\usepackage{subcaption}
\usepackage{tabularx}

\usepgfplotslibrary{fillbetween}

\DeclareMathOperator{\tr}{ {\rm tr} }
\DeclareMathOperator{\diag}{ {\rm diag} }

\DeclareMathOperator{\RF}{ {\rm RF} }
\DeclareMathOperator{\Gauss}{ {\rm Gauss} }

\newcommand{\RR}{\mathbb{R}}

\newcommand{\T}{ {\sf T} }
\newcommand{\EE}{\mathbb{E}}

\newcommand{\K}{{\mathbf{K}}}
\newcommand{\x}{{\mathbf{x}}}

\newcommand{\X}{{\mathbf{X}}}
\newcommand{\Y}{{\mathbf{Y}}}
\newcommand{\Z}{{\mathbf{Z}}}

\newcommand{\A}{{\mathbf{A}}}
\newcommand{\B}{{\mathbf{B}}}
\newcommand{\F}{{\mathbf{F}}}

\newcommand{\W}{{\mathbf{W}}}
\newcommand{\G}{{\mathbf{G}}}

\newcommand{\one}{{\mathbf{1}}}

\newcommand{\bOmega}{ \boldsymbol{\Omega} }

\newcommand{\bLambda}{ \boldsymbol{\Lambda} }
\newcommand{\bSigma}{ \boldsymbol{\Sigma} }

\newcommand{\C}{{\mathbf{C}}}
\newcommand{\I}{{\mathbf{I}}}

\newcommand{\uu}{{\mathbf{u}}}
\newcommand{\vv}{{\mathbf{v}}}

\renewcommand{\H}{{\mathbf{H}}}
\renewcommand{\L}{{\mathbf{L}}}

\newcommand{\bell}{{\boldsymbol{\ell}}}

\definecolor{RED}{rgb}{0.7,0,0}
\definecolor{BLUE}{rgb}{0,0,0.69}
\definecolor{GREEN}{rgb}{0,0.65,0}
\definecolor{PURPLE}{rgb}{0.69,0,0.8}
\definecolor{ORANGE}{rgb}{1,0.5,0}
\definecolor{ORANGERED}{rgb}{1.0,0.13,0.32}
\definecolor{BROWN}{rgb}{0.8,0.58,0.46}

\newcommand{\RED}{\color[rgb]{0.70,0,0}}
\newcommand{\BLUE}{\color[rgb]{0,0,0.69}}
\newcommand{\GREEN}{\color[rgb]{0,0.6,0}}
\newcommand{\PURPLE}{\color[rgb]{0.69,0,0.8}}
\newcommand{\ORANGE}{\color[rgb]{1,0.5,0}}

\newcommand{\BROWN}{\color[rgb]{0.8,0.58,0.46}}

\newtheorem{Theorem}{Theorem}

\newtheorem{Lemma}{Lemma}
\newtheorem{Remark}{Remark}
\newtheorem{Corollary}{Corollary}
\newtheorem{Definition}{Definition}

\begin{document}

\title{Robust and Communication-Efficient Federated Domain Adaptation via Random Features}

\author{
Zhanbo Feng\IEEEauthorrefmark{1},
\and
Yuanjie Wang\IEEEauthorrefmark{1},
\and
Jie Li,~\IEEEmembership{Fellow,~IEEE},
\and
Fan Yang,~\IEEEmembership{Member,~IEEE},
\and
Jiong Lou,~\IEEEmembership{Member,~IEEE},\\
\and
Tiebin Mi,~\IEEEmembership{Member,~IEEE},
\and
Robert.~C.~Qiu,~\IEEEmembership{Fellow,~IEEE},
\and
Zhenyu Liao\IEEEauthorrefmark{2},~\IEEEmembership{Member,~IEEE},
\thanks{
Z.~Feng J.~Li, F.~Yang, and J.~Lou are with the Department of Computer Science and Engineering (CSE), Shanghai Jiao Tong University, Shanghai, China.  
Y.~Wang, T.~Mi, R.~C.~Qiu, and Z.~Liao are with the School of Electronic Information and Communications (EIC), Huazhong University of Science and Technology, Wuhan, Hubei, China.
}
\thanks{Manuscript received xxxx, 2023; revised xxxx, 2024.}}

\markboth{Journal of \LaTeX\ Class Files,~Vol.~14, No.~8, August~2021}%
{Shell \MakeLowercase{\textit{et al.}}: A Sample Article Using IEEEtran.cls for IEEE Journals}


\maketitle

\begin{abstract}
Modern machine learning (ML) models have grown to a scale where training them on a single machine becomes impractical. 
As a result, there is a growing trend to leverage federated learning (FL) techniques to train large ML models in a distributed and collaborative manner. 
These models, however, when deployed on new devices, might struggle to generalize well due to domain shifts. 
In this context, federated domain adaptation (FDA) emerges as a powerful approach to address this challenge.

Most existing FDA approaches typically focus on aligning the distributions between source and target domains by minimizing their (e.g., MMD) distance.
Such strategies, however, inevitably introduce high communication overheads and can be highly sensitive to network reliability.

In this paper, we introduce RF-TCA, an enhancement to the standard Transfer Component Analysis approach that significantly accelerates computation without compromising theoretical and empirical performance. 
Leveraging the computational advantage of RF-TCA, we further extend it to FDA setting with FedRF-TCA. 
The proposed FedRF-TCA protocol boasts communication complexity that is \emph{independent} of the sample size, while maintaining performance that is either comparable to or even surpasses state-of-the-art FDA methods. 
We present extensive experiments to showcase the superior performance and robustness (to network condition) of FedRF-TCA.
\end{abstract}

\begin{IEEEkeywords}
Random features, maximum mean discrepancy, kernel method, federated domain adaptation
\end{IEEEkeywords}

\section{Introduction}
\label{sec:main}

\def\thefootnote{$*$}\footnotetext{Equal contribution.}
\def\thefootnote{$\dagger$}\footnotetext{Z.~Liao is the corresponding author (email: \href{mailto:zhenyu_liao@hust.edu.cn}{zhenyu\_liao@hust.edu.cn}).}

\renewcommand{\thefootnote}{\arabic{footnote}}

In today's context, marked by the unprecedented growth of machine learning (ML) and artificial intelligence (AI), increasingly larger ML models~\cite{brown2020language,rombach2022high} are being trained to address challenges across various fields, ranging from game~\cite{ye2020mastering,park2023generative}, text-to-image synthesis~\cite{saharia2022photorealistic}, natural language processing~\cite{brown2020language,openai2024gpt4technicalreport}, to robotics~\cite{karamcheti2023language} and intelligent decision-making~\cite{sutton2022quest}. 
The training of such colossal models is generally \emph{impossible} on a single machine. 
In this respect, federated learning (FL)~\cite{yang2019federated,mcmahan2017communication} techniques come into play, accelerating collaboration among different machines in training massive ML models while safeguarding data privacy across the network of participating distributed devices. 
Recent improvements have been focusing on on-device FL, which involves interactions within the network of distributed clients.
In this context, optimizing the communication costs across the network of distributed devices, ``aligning'' the (possibly very heterogeneous) datasets on local devices, and establishing reliable and efficient models are the major issues to address.

Models trained using FL techniques, however, may \emph{fail} to generalize on a new device due to \emph{domain shift}, in which case the distribution of the target (unlabeled) data significantly differs from that of the labeled data collected at source clients.
Unsupervised domain adaptation (DA)~\cite{wang2018deep}, in this respect, allows for  knowledge transferring from labeled source domains to unlabeled target domain, by aligning (in an unsupervised manner say) their feature distributions.

In this paper, we focus on the federated domain adaptation (FDA) approach to tackle the issue of domain shift in FL by minimizing their (e.g., maximum mean discrepancy, MMD~\cite{smola2007algo,gretton2006kernel}) distance through alternate information exchanges between devices.
These information exchanges, crucial to the performance of FDA methods, typically result in high communication overheads during the training procedure, see~\cite{peng2019federated,ghifary2014domain,sun2023feature}.
Moreover, in practical scenarios where the network of distributed machines is not always reliable -- potentially experiencing message drops and/or limited client availability (e.g., clients being temporarily unavailable and even dropping out during training)~\cite{kairouz2021Advances,li2020federated} -- existing FDA approaches  suffer from severe performance degradation.
In light of this, an asynchronous training scheme that is robust against (possibly limited) network reliability, without a significant sacrifice in performance and training time, is highly desirable in FDA.

\begin{figure*}
  \centering
  \includegraphics[width=140mm]{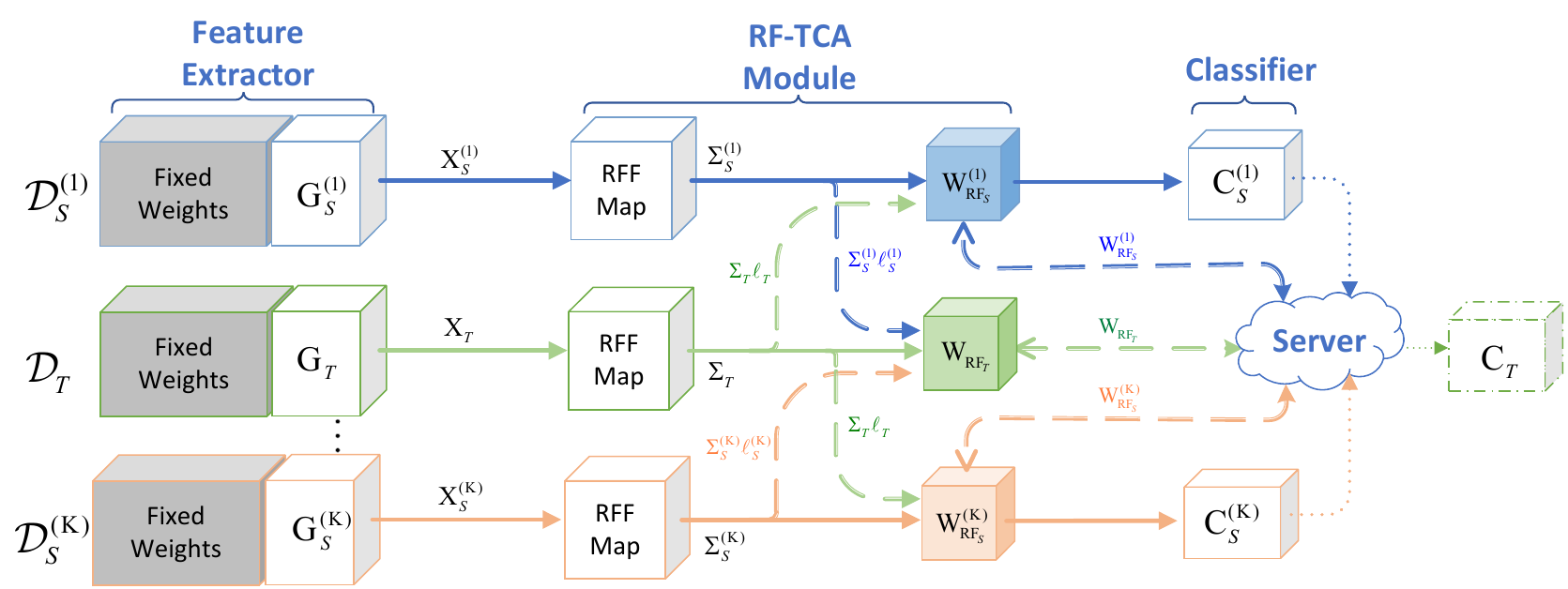}
  \caption{Illustration of the proposed FedRF-TCA protocol composed of
  (i) \textbf{Feature Extractor} with both fixed and learnable weights, denoted $\G_S$ and $\G_T$, respectively, obtained by fine-tuning a pretrained model (like ResNet-50); 
  (ii) \textbf{RF-TCA Transfer Module} using random features technique (see \Cref{def:RFF} below) and a linear adaptive layer ($\W_{\RF_S}^{(i)}$ or $\W_{\RF_T}$), with compressed features of the form $\bSigma_i \bell_i$ exchanged among clients during training; and
  (iii) \textbf{Classifier}.
  Solid arrows for local training and dashed arrows for global parameter aggregation between clients.
  See \Cref{sec:UFDA} for a detailed discussion.
  }
  \label{fig:multi_source_target_system}
\end{figure*}

\subsection{Our Approach and Contribution}
\label{subsec:our_approach}

In this paper, we propose a robust and communication-efficient FDA approach called Federated Random Features-based Transfer Component Analysis (FedRF-TCA), as depicted in \Cref{fig:multi_source_target_system}.
Aiming to minimize the communication overhead in training an FDA model, FedRF-TCA extends the random features approximation technique~\cite{rahimi2008random} to the Transfer Component Analysis (TCA)~\cite{pan2010domain} transfer learning scheme. 
This extension allows us to obtain compact source and target features with minimal MMD distance. 
From an FL perspective, the FedRF-TCA protocol is carefully designed so that only a \emph{minimal} amount of \emph{important} information is exchanged between clients in the network, in a highly compressed (and thus secure) fashion.
In a multi-source scenario, federated averaging (FedAvg) technique~\cite{mcmahan2017communication} is employed to train a shared classifier using data from different source domains.
This classifier is then applied on target features.

The primary contributions of this paper are as follows:
\begin{enumerate}
    \item We propose \textbf{RF-TCA}, an extension to standard TCA that significantly accelerates the vanilla TCA computations without theoretical \emph{and} empirical performance degradation (see, e.g., \Cref{theo:perf_RF_TCA} in \Cref{sec:RF-TCA}).
    This extension enables RF-TCA to be applied effectively on large-scale datasets, a key challenge for vanilla TCA.
    \item We introduce \textbf{FedRF-TCA}, a robust and efficient FDA approach. Unlike previous methods that suffer from high communications and/or computational costs on large-scale tasks, FedRF-TCA, as an extension of RF-TCA, enjoys a training communication overhead that is \emph{independent} of the dataset size and additional privacy protection, see \Cref{sec:UFDA}. 
    \item Extensive numerical experiments on various datasets are conducted in \Cref{sec:exper} to validate the efficiency of  computation, reduced communication overhead, and excellent robustness (to, e.g., network reliability) of FedRF-TCA.
\end{enumerate}

\subsection{Notations and Organization of the Paper}

We denote scalars by lowercase letters, vectors by bold lowercase, and matrices by bold uppercase. We denote the transpose operator by $(\cdot)^{\sf T}$, and use $\|\cdot\|_2$ to denote the Euclidean norm for vectors and spectral/operator norm for matrices. 
For a random variable $z$, $\mathbb{E}[z]$ denotes the expectation of $z$. 
We use $\one_p$ and $\I_{p}$ for the vector of all ones of dimension $p$ and the identity matrix of dimension $p\times p$, respectively.
For a given p.s.d.~matrix $\K$, we denote $\dim(\K) \equiv \tr \K/\| \K \|_2$ the intrinsic dimension of $\K$.
We use $\Theta, O$ and $\Omega$ notations as in classical computer science literature \cite{cormen2022introduction,de1981asymptotic}.

\medskip

The remainder of this paper is organized as follows.
In~\Cref{sec:previous} we review prior research efforts and revisit both the TCA and FedAvg approaches. 
Subsequently, in \Cref{sec:RF-TCA}, we present an extension of the original TCA transfer learning method, termed RF-TCA, which offers substantial reductions in both storage and computational complexity.
Building upon this, in \Cref{sec:UFDA}, we further extend RF-TCA to the FL scenario and introduce the FedRF-TCA protocol. 
FedRF-TCA not only ensures robust and communication-efficient training but also provides enhanced privacy protection. 
To validate the effectiveness of the proposed FedRF-TCA scheme, we present extensive experimental results on commonly used DA datasets in~\Cref{sec:exper}, demonstrating its superior performance compared to state-of-the-art methods.

\section{Previous Efforts and Preliminaries}
\label{sec:previous}

In this section, we provide a brief review of prior research endeavors concerning random features approximation, transfer learning, and FL in \Cref{subsec:review}. 
We will delve deeper into the Transfer Component Analysis (TCA) and Federated Averaging (FedAvg) approaches, and provide comprehensive insights into their methodologies in \Cref{subsec:MMD-TCA}~and~\ref{subsec:FederatedAve}, respectively.

\subsection{Review of Previous Efforts}
\label{subsec:review}

\subsubsection{Random features and kernel method}

Random features methods were first proposed to alleviate the computational and storage challenges associated with kernel methods, particularly in scenarios involving a large number of $n \gg 1$ data points~\cite{scholkopf2018kernel,rahimi2008random}.
Notably, Random Fourier Features (RFFs) are known to approximate the widely used Gaussian kernel, if a sufficiently large $N$ number of random features are used~\cite{rahimi2008random,liao2020rff}. 
Other random feature techniques have been developed to address more involved kernels, see~\cite{vedaldi2012efficient}. 
For an extensive overview of these approaches, interested readers are referred to~\cite{liu2021random}, and for their intricate connections with (deep and/or random) neural networks, we recommend exploring \cite{louart2018random,couillet_liao_2022,gu2022Lossless}.

\subsubsection{Transfer learning and domain adaptation}

Transfer learning, and in particular, domain adaptation, aims to ``align'' the distribution of source and target data representations in a common feature space. 
Different feature alignment strategies have been explored in the literature:
TCA~\cite{pan2011domain} and JDA~\cite{long2013transfer} propose to minimize the maximum mean discrepancy (MMD) distance between the source and target marginal/joint distributions.
GFK~\cite{gong2012geodesic} performs subspace feature alignment by exploiting the intrinsic low-dimensional structures of the data.
CORAL~\cite{sun2016return} proposes to align data representations via their second-order statistics.
DAN~\cite{long2015learning}, DaNN~\cite{ghifary2014domain}, and DDC~\cite{tzeng2014deep} leverage the strong expressive power of deep neural networks to align source and target representations in a reproducing kernel Hilbert space, based on MMD principle~\cite{smola2007algo,gretton2006kernel}.
In contrast, DANN~\cite{ganin2015unsupervised} addresses the domain shift through adversarial learning, upon which CDAN~\cite{long2018conditional} introduces conditioning strategies to further improve the performance.

\subsubsection{Federated learning and federated domain adaptation}

Federated learning (FL) is a decentralized ML paradigm that facilitates model training on local data while preserving privacy through the sharing of model updates across a network of devices.
Notwithstanding the remarkable progress achieved in the field of FL~\cite{yang2019federated,gu2020federated,smith2017federated}, the resulting trained models often falter in generalizing to data at novel devices, primarily due to \emph{domain shift}, where the feature distribution in target domain diverges from that in the source domain.
To address this challenge, federated domain adaptation (FDA) techniques have been proposed to mitigate domain shift by aligning the feature distributions of both source and target domains. 
While many FDA methods have demonstrated considerable empirical success, they often come at the cost of prohibitively high computational and communicational overheads~\cite{peng2019federated,song2020privacy}.
To alleviate the computational and communicational challenges inherent associated with FDA, recent work~\cite{sun2023feature,kang2022communicational} propose to compute MMD distance using a less resource-intensive feature extractor, thereby significantly reducing both computational and communicational overheads.
\cite{kang2022communicational} take this optimization a step further by introducing an approximate version of the original MMD loss, resulting in even more pronounced reductions in communication costs.

\medskip

In this work, we address the issue of domain shift challenge within FL and introduce the FedRF-TCA protocol. 
The proposed FedRF-TCA approach enjoys a much lower (in fact a data-size-\emph{independent}) communication complexity compared to existing FDA methods, and is, in addition, robust to network reliability and possible malicious (source) clients, all while delivering commendable performance.

\subsection{Maximum Mean Discrepancy Principle and Transfer Component Analysis}
\label{subsec:MMD-TCA}

The maximum mean discrepancy (MMD) was first introduced as a pivotal test statistic in the seminal works of \cite{smola2007algo,gretton2006kernel}. 
It serves as a tool to determine whether data points are drawn i.i.d.\@ from the same underlying distribution, via an evaluation of their features in a reproducing kernel Hilbert space (RKHS).
It quickly gains popularity as the preferred objective function for the purpose of ``aligning'' different feature distributions, from different domains, of particular interest to DA~\cite{ben-david2006Analysis}.
This further gave rise to numerous highly efficient transfer learning methods~\cite{pan2008Transfer,pan2010Survey}, one noteworthy approach among which is the Transfer Component Analysis (TCA) method~\cite{pan2011domain}.

TCA proposes to align the distributions of source $\x_S \in \RR^{p}$ and target data $\x_T \in \RR^p$ in some common feature space, via the following two-step transformation (visualized in \Cref{fig:space-TCA}):
\begin{enumerate}
  \item ``lift'' both the source and target data to some \emph{predefined} RKHS, via the kernel feature map $\phi(\cdot) \colon \RR^p \to \mathcal{H}$, to form the kernel matrix 
  \begin{equation}\label{eq:def_K}
    \K \equiv \langle \phi(\x_i), \phi(\x_j) \rangle_{i,j=1}^n \in \RR^{n \times n}, \quad n = n_S + n_T,
  \end{equation}
  with $\langle \cdot,  \cdot\rangle$ the inner product in $\mathcal{H}$;
  \item find, in an unsupervised manner say, a linear map $\tilde{\W} \in \RR^{n \times m}$ that ``projects'' the kernel features of both source and target data onto a low-dimensional space $\RR^m$ (with $m \ll n$ so that they can be computed or stored much more efficiently) according to the MMD principle, by solving the following optimization problem:
    \begin{align}
      \min_{\W }& \quad L_\gamma(\W) = \tr (\W^\T \K \bell \bell^\T \K \W) + \gamma \tr (\W^\T \W), \nonumber \\
      \text{s.t.}& \quad \W^\T \K \H \K \W = \I_m, \label{eq:vanilla_TCA} \tag{TCA}
    \end{align}
  with $\W = \K^{-1/2}\tilde{\W} \in \RR^{n \times m}$, regularization penalty $\gamma \geq 0$, $\H = \I_n - \one_n \one_n^\T/n \in \RR^{n \times n}$, and $\bell \in \RR^n $ the (normalized) source and target ``label'' vector $\bell \in \RR^{n}$ with its $i$th entry given by
  \begin{equation}\label{eq:def_ell}
    \ell_i = \frac1{n_S} 1_{\x_i \in \X_S} - \frac1{n_T} 1_{\x_i \in \X_T},
  \end{equation}
  for $\X_S\in \RR^{p \times n_S}$ the (set of) source data and $\X_T\in \RR^{p \times n_T}$ the (set of) target data, so that $\| \bell \|_2^2 = \frac{n}{n_S n_T}$.
  The solution to~\eqref{eq:vanilla_TCA} is explicitly given by $\W \in \RR^{n \times m}$, the $m$ top eigenvectors (that correspond to the largest $m$ eigenvalues) of $(\gamma \I_n + \K \bell \bell^\T \K)^{-1} \K \H \K$ satisfying
  \begin{equation}\label{eq:TCA_solution_eigen}
    (\gamma \I_n + \K \bell \bell^\T \K)^{-1} \K \H \K \cdot \W =  \W \cdot \bLambda
  \end{equation}
  with diagonal $\bLambda = \diag \{ \lambda_i \}_{i=1}^m$ containing the largest $m$ eigenvalues of $(\gamma \I_n + \K \bell \bell^\T \K)^{-1} \K \H \K$, with the obtained low-dimensional and ``aligned'' representations (of both source and target data) given by the columns of $\W^\T \K \in \RR^{m \times n}$.
\end{enumerate}
This two-step transformation from $\RR^p$ to RKHS $\mathcal H$ and then to $\RR^m$ is visualized in \Cref{fig:space-TCA}, where we denote $\X_{\RR^p}, \X_{\mathcal H}, \X_{\RR^m}$ the features in different spaces, and $\K_{\RR^p}, \K_{\mathcal H}, \K_{\RR^m}$ the corresponding kernel matrices, respectively.

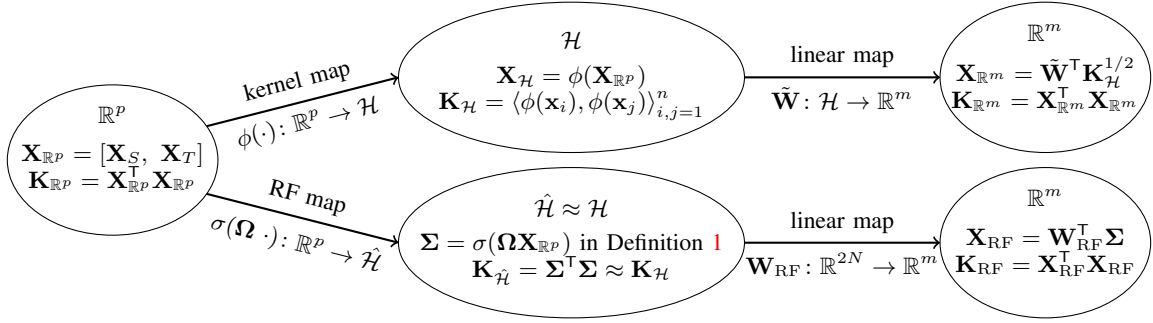
\begin{figure*}
  \centering
  \begin{tikzpicture}[font=\footnotesize]
      \tikzstyle{every node}=[font=\small]
      \draw (-6.1,-1.1) circle [x radius=1.4cm, y radius=10mm];
      \node at (-6.1,-0.5) {$ \RR^p $};
      \node at (-6.1,-1.2) [align=center] {$ \X_{\RR^p} = [\X_S,~\X_T]$ \\ $\K_{\RR^p} = \X_{\RR^p}^\T \X_{\RR^p}$};
      
      \draw (0,0) circle [x radius=2.3cm, y radius=10mm];
      \node at (0,0.5) {$\mathcal{H}$};
      \node at (0,-0.2) [align=center] {$ \X_{\mathcal H} = \phi(\X_{\RR^p}) $ \\ $\K_{\mathcal H} = {\langle \phi(\x_i),\phi(\x_j) \rangle}^n_{i,j=1}$};
      
      \draw (6.3,0) circle [x radius=1.4cm, y radius=10mm];
      \node at (6.3,0.6) {$ \RR^m $};
      \node at (6.3,-0.1) [align=center] {$ \X_{\RR^m} = \tilde{\W}^\T \K_{\mathcal H}^{1/2} $ \\ $\K_{\RR^m} = \X_{\RR^m}^\T \X_{\RR^m}$};
      
      \draw [thick,->] (-4.85,-0.65) -- (-2.3,0) node [midway, above, sloped] { kernel map} node [midway, below, sloped] {$\phi(\cdot) \colon \RR^p \to \mathcal{H}$};
      \draw [thick,,->] (2.3,0) -- (4.9,0) node [midway, above, sloped] {linear map} node [midway, below, sloped] {$\tilde{\W} \colon \mathcal{H} \to \RR^m $};
  
      \draw (0,-2.2) circle [x radius=2.3cm, y radius=10mm];
      \node at (0,-1.7) {$\hat{\mathcal{H}} \approx \mathcal{H} $}; 
      \node at (0,-2.4) [align=center] {$ \bSigma = \sigma(\bOmega \X_{\RR^p})$ in \Cref{def:RFF} \\ $\K_{\hat{\mathcal{H}}} = \bSigma^\T \bSigma \approx \K_{\mathcal H} $  }; 
      
      \draw (6.3,-2.2) circle [x radius=1.4cm, y radius=10mm];
      \node at (6.3,-1.6) {$ \RR^m $};
      \node at (6.3,-2.3) [align=center] {$ \X_{\RF} = \W_{\RF}^\T \bSigma $ \\ $\K_{\RF} = \X_{\RF}^\T \X_{\RF}$};
      
      \draw [thick,->] (-4.85,-1.55) -- (-2.3,-2.2) node [midway, above, sloped] { RF map} node [midway, below, sloped] {$\sigma(\bOmega~\cdot) \colon \RR^p \to \hat{\mathcal{H}} $};
      \draw [thick,,->] (2.3,-2.2) -- (4.9,-2.2) node [midway, above, sloped] {linear map} node [midway, below, sloped] {$\W_{\RF} \colon \RR^{2N} \to \RR^m $};
    \end{tikzpicture}
    \caption{ { \textbf{Top}: Two-step ``transformation'' in TCA, from raw data space $\RR^p$ to (possibly infinite-dimensional) RKHS $\mathcal H$, and then to the low-dimensional $\RR^m$. \textbf{Bottom}: Two-step ``transformation'' in the proposed RF-TCA, from $\RR^p$ to random features kernel space $\hat{\mathcal{H}} \subset \RR^{2N}$, and then to the low-dimensional $\RR^m$.  }}
  \label{fig:space-TCA}
  \end{figure*}

To distinguish the original TCA approach in \eqref{eq:vanilla_TCA} from its forthcoming random-features-based variant, we will henceforth refer to it as \emph{vanilla} TCA.

Notably, it is worth mentioning that the matrix inverse $(\gamma \I_n + \K \bell \bell^\T \K)^{-1}$, being a rank-one update of the inverse of $\gamma \I_n$, can be \emph{explicitly} computed using the Sherman--Morrison formula (see \Cref{lem:sherman-morrison} in \Cref{sec:app_lemmas}).
This is discussed in the following result.
\begin{Lemma}[Equivalent form of vanilla TCA]\label{lem:equivalent_form}
The matrix $\gamma \I_n + \K \bell \bell^\T \K$ is invertible \emph{if and only if} $\gamma + \bell^\T \K^2\bell \neq 0$, and one has $ (\gamma \I_n + \K \bell \bell^\T \K )^{-1} = \frac1{\gamma} ( \I_n - \frac{ \K \bell \bell^\T \K }{\gamma + \bell^\T \K^2 \bell} )$, so that the ``transformed'' feature $\H \K \W$ obtained from vanilla TCA, with $\W$ solution to \eqref{eq:vanilla_TCA}, is given by the top eigenspace that corresponds to the largest $m$ eigenvalues of
\begin{equation}\label{eq:def_A}
  \A \equiv \H \left( \K^2 - \frac{ \K^2 \bell \bell^\T \K^2 }{\gamma + \bell^\T \K^2 \bell} \right) \H.
\end{equation}
\end{Lemma}
The advantage of the formulation in \Cref{lem:equivalent_form} is that one no longer needs to invert the $n$-by-$n$ matrix (that takes $O(n^3)$ time in general) for vanilla TCA in \eqref{eq:TCA_solution_eigen}, but only to perform matrix additions and matrix-vector multiplication via, e.g., Lanczos iteration that takes $O(n^2)$ time \cite{golub2013matrix} to retrieve the top eigenvectors.

In \Cref{sec:appendix_fact_TCA} of the appendix, we provide further discussions and numerical results on vanilla TCA.

\subsection{Federated Averaging}
\label{subsec:FederatedAve}

Model aggregation is a key enabler in federated learning (FL) that allows for collaborative training of large-scale ML models by, e.g., combining the model parameters from different clients.
In this respect, FedSGD and federated averaging (FedAvg)~\cite{mcmahan2017communication} allows users to collaboratively benefit from shared models trained on diverse data without central storage.
FedMA~\cite{wang2020federated}, on the other hand, proposes to construct a shared ML model (e.g., a deep neural network) in a layer-by-layer manner by ``matching'' and averaging the hidden elements.

In the proposed FedRF-TCA protocol, we utilize federated averaging to update the model parameters in both the classifier $\C$ and the linear adaptive layer $\W_{\RF}$ in~\Cref{fig:multi_source_target_system}, due to its simplicity and low communication cost.
Federated averaging (FedAvg) was introduced in~\cite{mcmahan2017communication} to train a ML model in a decentralized way by iteratively aggregating locally computed updates.
Consider a FL scenario where the training data are available at $K$ clients, the FedAvg protocol proposes, in each iteration, to first locally update the model parameter at each client and then send the updated parameter to a central server. 
The server then aggregates the model parameters $\W_i, i \in \{1, \ldots, K\}$ from the selected clients. 
FedAvg allows one to collectively reap the benefits of shared models trained on a large number of data partitioned over $K$ clients, without centrally storing these data in a server. 

\subsection{Federated Domain Adaptation}
\label{subsec:FederatedDA}

Federated Domain Adaptation (FDA) aims to resolve the non-i.i.d.~and domain shift issues of data in a FL context~\cite{peng2019federated, kang2022communicational}. 
Due to the different origin of datasets at different clients in a  FL system, there may exist domain shift between them~\cite{quinonero2022dataset}. 
FDA proposes to address this domain shift issue by transferring knowledge from the source clients to a novel target client, in an unsupervised fashion, as follows.
\begin{Definition}[Unsupervised Federated Domain Adaptation, UFDA]
  Consider a federated learning system having $K $ clients.
  Denote $\mathcal{D}_S^{(i)}, i \in \{1,\dots, K\}$ the source domains at these $K$ clients, in the sense that the data and label pairs satisfy $(\X_S^{(i)}, \Y_S^{(i)}) \sim \mathcal{D}_S^{(i)}$. 
  The problem of Unsupervised Federated Domain Adaptation (UFDA) aims to decide on (e.g., the label $\Y_T$ of) the unlabelled data $\X_T$ at a novel target client in the system from a different domain $(\X_T, \Y_T) \sim \mathcal{D}_T$, via federated information exchange in the system.
  \end{Definition}

\section{A Random Features Approach to TCA}
\label{sec:RF-TCA}

In this section, we propose RF-TCA, a random features-based computationally efficient approach to TCA.
Under standard assumptions, we demonstrate in \Cref{theo:perf_RF_TCA} that RF-TCA substantially mitigates the computational overhead associated with vanilla TCA, while exhibiting (virtually) no degradation in transfer learning performance compared to vanilla TCA.

\subsection{RF-TCA: A Random Features Approach to Efficient TCA}
\label{subsec:RF-TCA}

Here, we present RF-TCA, a random features-based approach to computationally efficient TCA, by focusing on random Fourier features (RFFs) and Gaussian kernel.
We first recall the definition of RFFs as follows.

\begin{Definition}[Random Fourier features for Gaussian kernels,~\cite{rahimi2008random}]\label{def:RFF}
For data matrix $\X = [\x_1, \ldots, \x_n] \in \RR^{p \times n}$ of size $n$, the random Fourier feature (RFF) matrix $\bSigma \in \RR^{2N \times n}$ of $\X$ is given by\footnote{Extensions of RFFs to other shift-invariant kernels (such as the Laplacian and Cauchy kernels) also exists, see~\cite{rahimi2008random}.}
\begin{equation}\label{eq:def_RFF}
  \bSigma = \frac{1}{\sqrt N} \begin{bmatrix} \cos(\bOmega \X) \\ \sin(\bOmega \X) \end{bmatrix} \in \RR^{2N \times n},
\end{equation}
with $N$ the number of random features, $\bOmega \in \RR^{N \times p}$ a random matrix having i.i.d.~standard Gaussian entries with mean zero and variance $1/\sigma^2$, i.e., $[\bOmega]_{ij} \overset{i.i.d.}{\sim} \mathcal N(0,1/\sigma^2)$, and nonlinear functions $\cos(\cdot), \sin(\cdot)$ applied entry-wise on $\bOmega \X$.
\end{Definition}
\noindent
Denote $\dim(\K) \equiv \tr \K/\| \K \|_2$ the intrinsic dimension of the Gaussian kernel matrix $\K = \K_{\Gauss}$, it is known, e.g., per \Cref{theo:RFF_approx_spectral} below, that a number of $N = O( \dim(\K) \log(n))$ RFFs suffices to well approximate the Gaussian kernel matrix $\K$ in a spectral norm sense $\bSigma^\T \bSigma \approx \K$.
This allows one to obtain an \emph{effective low-rank approximation} of $\K$, and can be further exploited to propose RF-TCA that significantly reduces the computational burden of vanilla TCA.

Precisely, following the idea of vanilla TCA in \eqref{eq:vanilla_TCA}, we aim to find a matrix $\W_{\RF} \in \RR^{2N \times m}$ that ``projects'' the RFFs $\bSigma$ onto an $m$-dimensional space, with $m \ll \min(2N,n)$, to obtain the ``transferred and aligned'' representations $\W_{\RF}^\T \bSigma \in \RR^{m \times n}$ with kernel matrix $\tilde \K = \bSigma^\T \W_{\RF} \W_{\RF}^\T \bSigma$ that minimizes the MMD loss.
Such projection $\W_{\RF} $ can be obtained by solving the following optimization problem:
\begin{align}
  \min_{\W \in \RR^{2N \times m}} &\quad \tr (\W^\T \bSigma \bell \bell^\T \bSigma^\T \W) + \gamma \tr (\W^\T \W)  \nonumber \\ 
  \text{s.t.} &\quad \W^\T \bSigma \H \bSigma^\T \W = \I_m, \label{eq:RF_TCA} \tag{RF-TCA}
\end{align}
with regularization $\gamma $ and $\H$ as for vanilla TCA in \eqref{eq:vanilla_TCA}.

Assume $\bSigma \H \bSigma^\T \in \RR^{2N \times 2N}$ is invertible and let $\tilde \W \equiv (\bSigma \H \bSigma^\T)^{\frac12} \W \in \RR^{2N \times m}$, the problem \eqref{eq:RF_TCA} then writes
\begin{align}
  \min &\quad \tr \left(\tilde \W^\T (\bSigma \H \bSigma^\T)^{-\frac12} (\bSigma \L \bSigma^\T + \gamma \I_{2N}) (\bSigma \H \bSigma^\T)^{-\frac12} \tilde \W \right) \nonumber \\ 
  \text{s.t.} & \quad \tilde \W^\T \tilde \W = \I_m,
\end{align}
which also takes the form of a trace minimization, see~\cite[Proposition~1]{pan2011domain} and \cite{von2007tutorial}. 
It then follows from the Rayleigh--Ritz theorem that the optimal solution $\W_{\RF}$ is given by the top $m$ eigenvectors of $ (\bSigma \bell \bell^\T \bSigma^\T + \gamma \I_{2N})^{-1} \bSigma \H \bSigma^\T$ associated to the $m$ largest eigenvalues. 

Also, as in \Cref{lem:equivalent_form} for vanilla TCA, $(\bSigma \bell \bell^\T \bSigma^\T + \gamma \I_{2N})^{-1}$ takes an explicit form per the Sherman--Morrison formula and can be computed with ease.
This leads to the RF-TCA approach in \Cref{alg:RF-TCA}. 

\begin{algorithm}[htb]
\caption{ Random features-based TCA (RF-TCA) }
\label{alg:RF-TCA}
\begin{algorithmic}[1]
\STATE {\bfseries Input:} Source data $ \X_S \in \RR^{p \times n_S}$, target data $ \X_T \in \RR^{p \times n_T}$, number of random features $N$, and projected dimension $m$.
\STATE {\bfseries Output:} ``Aligned'' features $\F_{\RF} = [\F_S,~\F_T]$ with $ \F_S \in \RR^{m \times n_S}, \F_T \in \RR^{m \times n_T}$ of $\X_S$ and $\X_T$, respectively.
\STATE\hspace{1em}Compute RFFs $\bSigma \in  \RR^{2N \times n}$ of the source and target data $\X = [\X_S,~\X_T] \in \RR^{p \times n}, n = n_S + n_T$ as in \eqref{eq:def_RFF}.
\STATE\hspace{1em}Compute $\W_{\RF} \in \RR^{2N \times m}$ as the $m$ dominant eigenvectors (that correspond to the largest $m$ eigenvalues) of
\begin{equation}
  \bSigma \H \bSigma^\T - \frac{ \bSigma \bell \bell^\T \bSigma^\T \bSigma \H \bSigma^\T }{ \gamma + \bell^\T \bSigma^\T \bSigma \bell },
\end{equation}
for $\H = \I_n -  \one_n \one_n^\T/n $ and $\bell \in \RR^n$ defined in \eqref{eq:def_ell}.
\STATE \textbf{return} RF-TCA ``transferred'' features as $\F_{\RF} = \W_{\RF}^\T \bSigma$.
\end{algorithmic}
\end{algorithm}

\subsection{Performance Guarantee of RF-TCA}
\label{subsec:R-TCA}

In the following result, we show that with the proposed RF-TCA approach in \Cref{alg:RF-TCA}, an order of $\dim(\K) \log(n)$ random features suffices to match the performance of vanilla TCA in the sense of MMD loss, up to a proper tuning of hyperparameter $\gamma$.

\begin{Theorem}[Performance guarantee for RF-TCA]\label{theo:perf_RF_TCA}
Let random Fourier features matrix $\bSigma \in \RR^{2N \times n}$ be defined in \Cref{def:RFF} with respect to Gaussian kernel matrix $\K = \K_{\Gauss} \in \RR^{n \times n}$ having maximum and minimum eigenvalue $\lambda_{\max}(\K) \geq \lambda_{\min}(\K) >0 0$. Let $\W_{\gamma} \in \RR^{n \times m}$ be the vanilla TCA solution to \eqref{eq:vanilla_TCA} with regularization $\gamma$, and let $\W_{\RF} \in \RR^{2N \times m}$ be the solution to \eqref{eq:RF_TCA}.
Then, for any given $\varepsilon \in (0, 1 )$, there exists $\W \in \RR^{n \times m}$ having MMD loss upper and lower bounded by that of $\W_{\gamma\lambda_{\min}}$ and $\W_{\gamma\lambda_{\max}}$, for which we have
\begin{equation}
  \EE \| \H \bSigma^\T \W_{\RF} - \H \K \W \|_F \leq \varepsilon,
\end{equation}
holds as long as $N \geq C \dim(\K) \log(n)/ \varepsilon^2 $ for some positive constant $C$ that depends on the \emph{relative eigen-gap} $\tilde{\Delta}_\lambda \equiv \min_{1 \leq i \leq m} \left|\lambda_i(\A_{\rm R})-\lambda_{i+1}(\A_{\rm R}) \right|/ \| \K \|_2 > 0$, with $\lambda_i(\A_{\rm R})$ the eigenvalues of $\A_{\rm R} \equiv \H \K (\gamma \K + \K \bell \bell^\T \K )^{-1} \K \H$ listed in a numerically decreasing order.

\end{Theorem}

\begin{Remark}[Eigen-gap condition]\label{rem:eign-gap}\normalfont
Note that the bound in \Cref{theo:perf_RF_TCA} becomes vacuous when the \emph{relative} eigen-gap condition $\tilde{\Delta}_\lambda \equiv \min_{1 \leq i \leq m} \left|\lambda_i(\A)-\lambda_{i+1}(\A) \right|/\| \K \|_2 > 0$ is violated.
This condition is needed in the application of the Davis--Kahan theorem and is rather standard in the literature of spectral methods, see also~\cite{joseph2016impact,von2007tutorial}.
\end{Remark}
\noindent
\Cref{theo:perf_RF_TCA} states, under the eigen-gap condition in \Cref{rem:eign-gap}, that $N = O( \dim(\K) \log (n))$ random features are enough to obtain RF-TCA transferred features as an accurate ``proxy'' of the those from vanilla TCA, up to regularization.

\begin{proof}[Proof of \Cref{theo:perf_RF_TCA}]
To prove \Cref{theo:perf_RF_TCA}, we need to introduce another TCA-type method that can, as we shall see, connect the performance of vanilla to that of RF-TCA.
This is done by penalizing the ``projector'' $\tilde \W \colon \mathcal H \to \RR^p$ on the RKHS, as opposed to regularizing the Frobenius norm of the linear map $\W \in \RR^{n \times m}$ in the case of vanilla TCA in \Cref{subsec:MMD-TCA}.
This approach, referred to as Regularized TCA (R-TCA) in the subsequent section, formulates the following optimization problem:
\begin{align}
  \min_{\W} &\quad \tilde L_\gamma(\W)= \tr (\W^\T \K \bell \bell^\T \K \W) + \gamma \tr (\W^\T \K \W), \nonumber \\
  \text{s.t.}&\quad \W^\T \K \H \K \W = \I_m, \label{eq:alternative_TCA} \tag{R-TCA}
\end{align}
and differs from \eqref{eq:vanilla_TCA} \emph{only} in the regularization term.
An inspection reveals the following close relationship between the objective functions of the two approaches.
\begin{Lemma}[TCA versus R-TCA]\label{lem:TCA-VS-R-TCA}
Let the kernel matrix $\K$ be defined as in \eqref{eq:def_K} and let $\lambda_{\max}(\K) \geq \lambda_{\min}(\K) \geq 0$ be its maximum and minimum eigenvalue, respectively. 
Then, for any given $\W \in \RR^{n \times m}$, one has
\begin{align*}
 L_{\gamma \lambda_{\min}(\K) } (\W) \leq  \tilde L_{\gamma}(\W) \leq L_{\gamma \lambda_{\max}(\K) } (\W).
\end{align*}
\end{Lemma}
\noindent
The proof of \Cref{lem:TCA-VS-R-TCA} follows straightforwardly form the fact that $ 0 \leq \lambda_{\min}(\K) \tr (\W^\T \W) \leq \tr (\W^\T \K \W) \leq \lambda_{\max} (\K) \tr (\W^\T \W)$.
It thus becomes evident that the two optimization problems in \eqref{eq:vanilla_TCA} and \eqref{eq:alternative_TCA} are ``equivalent'' concerning their loss functions, differing only by a scaling factor on the regularization parameter $\gamma$.

As for vanilla TCA in \eqref{eq:TCA_solution_eigen}, the R-TCA approach in \eqref{eq:alternative_TCA} also admits an \emph{explicit} solution $\W_{\rm R}$ , given by the top $m$ eigenvector of $(\gamma \K + \K \bell \bell^\T \K )^{-1} \K \H \K$ (the proof of which is the same as \cite[Proposition~1]{pan2011domain} and is omitted), that is
\begin{equation}\label{eq:TCA_solution_eigen_sym}
  \H \K (\gamma \K + \K \bell \bell^\T \K )^{-1} \K \H \cdot \H \K \W_{\rm R} =  \H \K \W_{\rm R} \cdot \bLambda_{\rm R},
\end{equation}
where we denote $\bLambda_{\rm R}$ the diagonal matrix containing the $m$ largest eigenvalues of $(\gamma \K + \K \bell \bell^\T \K )^{-1} \K \H \K$.

To finish the proof of \Cref{theo:perf_RF_TCA}, it remains to compare the solution of R-TCA and RF-TCA, for an order of $N = O( \dim(\K) \log(n))$ RFFs.
We need the following result.

\begin{Theorem}[RFFs approximation of Gaussian kernels,~{\cite[Section~6.5]{tropp2015matrix}}] 
\label{theo:RFF_approx_spectral}
For random Fourier features $\bSigma \in \RR^{2N \times n}$ of data $\X \in \RR^{p \times n}$ as defined in \Cref{def:RFF}, one has
\begin{equation*}
  \EE \| \bSigma^\T \bSigma - \K \|_2 \leq C \left( \sqrt{ \frac{n \log(n)}N  \| \K  \|_2 } +  \frac{n \log(n)}N \right),
\end{equation*}
holds for some universal constant $C > 0$ independent of $N$ and $n$, with $\K = \K_{\Gauss} = \{ \exp(- \| \x_i - \x_j \|^2/ (2\sigma^2) ) \}_{i,j=1}^n$ the Gaussian kernel matrix of $\X$.
\end{Theorem}
\noindent
Given $\varepsilon \in (0,1)$, taking $N \geq C' \dim(\K) \log(n)/ \varepsilon^2 $ for some universal constant $C' > 0$ independent of $n$, with $\dim(\K) \equiv \tr \K/\| \K \|_2  = n/ \| \K \|_2 $ the \emph{intrinsic dimension} of Gaussian kernel matrix $\K = \K_{\Gauss}$, it follows from \Cref{theo:RFF_approx_spectral} that the \emph{expected} relative error satisfies
\begin{equation}
   \EE  \| \bSigma^\T \bSigma - \K \|_2 / \| \K \|_2  \leq \varepsilon,
\end{equation}
and it suffices to have an order of $\dim(\K) \log(n)$ RFFs so that the RFF Gram matrix $\bSigma^\T \bSigma $ is a good approximation of the Gaussian kernel matrix in a \emph{spectral norm} sense.

With the spectral norm approximation in \Cref{theo:RFF_approx_spectral}, the proof of \Cref{theo:perf_RF_TCA} follows from basic algebraic manipulations.
We refer the readers to \Cref{sec:proof_of_theo_RF_TCA} for details.
\end{proof}

\medskip

We have seen in this section that the proposed RF-TCA approach in \Cref{alg:RF-TCA} enjoys the numerical (and thus communicational in a FL context) advantage of operating on a much lower dimensional space and at the same time, obtains ``transferred'' features close to those of TCA per \Cref{theo:perf_RF_TCA}.
Following our framework in \Cref{fig:multi_source_target_system} of \Cref{subsec:our_approach}, we discuss next how the proposed RF-TCA serves as a cornerstone to the robust and communication-efficient FedRF-TCA approach, and, at the same time, provides additional privacy guarantee.

\section{Robust and Communication-efficient Federated Domain Adaptation via RF-TCA}
\label{sec:UFDA}

The majority, if not all, of the DA methods based on MMD, while known for their effectiveness in aligning feature distributions between source and target data, impose a substantial computational burden. 
This computational complexity primarily arises from the necessity of operations involving the \emph{large} kernel matrices of both source and target data, as highlighted in prior works~\cite{long2015learning,tzeng2014deep}.
This makes them unsuitable for use in an FL context, for which low communication and computational costs are desired~\cite{sun2023feature}.

In response to this challenge, we propose Federated Random Features-based Transfer Component Analysis (FedRF-TCA), a novel FDA approach designed to alleviate the computational and communicational demands and enhance privacy protection. 
FedRF-TCA achieves this by employing compression and randomization techniques during the (asynchronous) training procedure, leveraging the RF-TCA methodology introduced in \Cref{sec:RF-TCA}, to optimize information exchange within the MMD-based FDA scheme. 
In comparison to existing MMD-based FDA protocols, FedRF-TCA offers significant reductions in both communication and computational overheads while maintaining commendable performance.

Specifically, we are primarily focused on a federated multi-source classification task, which involves multiple source clients coexisting alongside one target client. 
Each source client is characterized by a distinct domain, comprising its own dataset and corresponding labels. 
In contrast, the target client solely possesses data without accompanying labels. 
The overarching goal of FDA is to train a model that can effectively classify the target data while also addressing concerns such as data leakage.

The FedRF-TCA protocol comprises a structured framework for each client, encompassing three distinct local components, as illustrated in \Cref{fig:multi_source_target_system}:
\begin{itemize}
    \item[(i)] \textbf{Feature Extractor}: For each source and target domain, an dedicated  local feature extractor is trained by fine-tuning a pretrained model, such as ResNet-50. 
    These feature extractors are denoted as $\G_S^{(1)}, \ldots , \G_S^{(K)}$ for the $K$ source domains $\mathcal{D}_S^{(1)}, \ldots, \mathcal{D}_S^{(K)}$, and $\G_T$ for the target domain $\mathcal{D}_T$ in \Cref{fig:multi_source_target_system}.
    This training process yields local and ``unaligned'' source features $\X_S^{(1)} \in \RR^{p \times n_S^{(1)}}, \ldots, \X_S^{(K)} \in \RR^{p \times n_S^{(K)}}$ and target features $\X_T \in \RR^{p \times n_T}$, respectively.
    \item[(ii)] \textbf{RF-TCA Transfer Module}: In this module, the random feature map in \Cref{def:RFF} is applied on $\X_S^{(1)}, \ldots, \X_S^{(K)}$ and $\X_T$ to intermediate representations $\bSigma_S^{(1)}, \ldots, \bSigma_S^{(K)}$ and $\bSigma_T$. 
    Subsequently, an \emph{adaptive} linear layer $\W_{\RF}$ is applied to these RFFs, aligning the source and target feature distributions, by minimizing their MMD distance.
    \item[(iii)] \textbf{Classifier}: Finally, source classifiers $\C_S^{(i)}$ are trained on the ``aligned'' source features from RF-TCA module, and their outputs are aggregated to obtain the target classifier, serving the desired classification purpose.
\end{itemize}

In the remainder of section, we discuss in \Cref{subsec:sec_arch_RF} how, for a given pair of source and target domain $(\mathcal{D}_S^{(i)}, \mathcal{D}_T)$, features are extracted from data in distinct domains, and ``transferred'' into a common domain via random feature maps and a trainable linear layer $\W_{\RF}$; 
and then in \Cref{subsec:sec_arch_net} how, in a multi-source FDA setting, the local models of source clients are updated and subsequently aggregated to construct the target model.
Finally, in \Cref{subsec:advanteges}, we undertake a comprehensive analysis of the advantages conferred by the proposed FedRF-TCA protocol, in terms of its computational and communication overhead, in addition to its enhanced robustness as well as privacy protection mechanisms.

\subsection{MMD-based Random Features Alignment} 
\label{subsec:sec_arch_RF}

To minimize the MMD distance between source and target feature distributions, FedRF-TCA proposes to use 
\begin{itemize}
  \item[(i)] a (local) feature extractor $\G_S$; followed by
  \item[(ii-i)] a random features ``compressor'' as in \Cref{def:RFF}; and 
  \item[(ii-ii)] a linear layer $\W_{\RF} \in \mathbb{R}^{2N \times m}$ that aligns features from source and target domains. 
\end{itemize}
Specifically, for a given pair of source-target domain $(\mathcal{D}_S^{(i)}, \mathcal{D}_T)$, $i \in \{ 1, \ldots, K\}$, we define the following MMD loss between the $i$th source domain $\mathcal{D}_S^{(i)}$ and the target domain $\mathcal{D}_T$ as
\begin{align}
    &L_{\rm MMD}^{(i)} (\G_S^{(i)}, \G_T, \W_{\RF}) \label{eq:def_mmd_loss} \\ 
    &= ( \bSigma_S^{(i)} \bell_S^{(i)} + \bSigma_T \bell_T )^\T \W_{\RF} \W_{\RF}^\T ( \bSigma_S^{(i)} \bell_S^{(i)} + \bSigma_T \bell_T ), \nonumber 
\end{align}
in the same spirit as in \eqref{eq:RF_TCA} for RF-TCA but without the regularization term, where the RF matrices $\bSigma_S^{(i)}$ and $\bSigma_T$ are computed as in \Cref{def:RFF} with input $\X_S^{(i)}$ and $\X_T$, respectively, $\bell_S^{(i)} = \one_{n_S^{(i)}}/n_S^{(i)}$ for data from the $i$th source domain and $\bell_T = - \one_{n_T}/n_T$ for target from target domain as in \eqref{eq:def_ell}, with $N$ the number of random features. 

Note that the MMD loss in \eqref{eq:def_mmd_loss} is defined for the $i$th source domain $\mathcal{D}_S^{(i)}$ and the target domain $\mathcal{D}_T$, and depends on the source and target \emph{local} features extractors $\G_S^{(i)}$ and $ \G_T$, as well as the local linear (RF-based) feature ``aligner'' $\W_{\RF}$ at source or target client.

The proposed MMD formulation in \eqref{eq:def_mmd_loss} extends the RF-TCA thoroughly approach disused in \Cref{sec:RF-TCA}, by relaxing the second-order constraint $\W_{\RF}^\T \bSigma \H \bSigma^\T \W_{\RF} = \I_m$ in \eqref{eq:RF_TCA}, and thereby allows for \emph{learnable} linear feature aligner $\W_{\RF}$ \emph{and} local feature extractors $\G_S, \G_T$.
This makes the MMD minimization in \eqref{eq:def_mmd_loss} more flexible, e.g., by avoiding the need of simultaneous access to \emph{all} source and target data.
Notably, the linear layer $\W_{\RF}$ can now be improved (e.g., via backpropagation) by exploiting the information in a mini-batch of source and target data, a property of particular interest to federated learning.

In the context of multi-source federated domain adaptation, the frequent message passing between clients results in communication overhead and imposes robustness requirements on the network.
In this respect, the proposed formulation in \eqref{eq:def_mmd_loss} is both communication-efficient and robust in that:

In the context of multi-source federated domain adaptation, the frequent exchange of messages between clients can lead to significant communication overhead and impose robustness requirements on the network infrastructure.
In this context, our proposed formulation in \eqref{eq:def_mmd_loss} offers notable advantages in terms of communication efficiency and robustness:
\begin{enumerate}
  \item \textbf{Efficient Communication}: Messages shared between source and target clients are highly compressed and take the form of summed features as $\bSigma \bell \in \RR^{2N}$. 
  Importantly, the size of such messages is solely dependent on the number of random features $N$, and is \emph{independent} of the (source and/or target) sample size.
  \item \textbf{Robust Asynchronous Training}: The federated domain adaptation process can be performed in an asynchronous manner.
  This is made possible due to the proposed decomposable MMD loss in \eqref{eq:def_mmd_loss}, for a specific pair of source-target domains $(\mathcal{D}_S^{(i)}, \mathcal{D}_T)$. Consequently, the FDA system can continue to improve during the training/transferring procedure, even in the presence of stragglers or irregularities in client participation.
\end{enumerate}

\subsection{FedRF-TCA: multi-source FDA via RF-TCA and FedAvg}
\label{subsec:sec_arch_net}

While most existing DA methods are designed for a single pair of source-target domain, here we show how the proposed FedRF-TCA protocol naturally extends to a multi-source FDA scenario. 

Consider $K$ source domains $\mathcal{D}_S^{(i)}$ with associated data and labels $( \X_S^{(i)}, \Y_S^{(i)} ),  i \in \{1, \ldots, K \}$, and a single target domain $\mathcal{D}_T$ with unlabelled data $\X_T$, upon which some decision (e.g., classification) should be made in an FDA fashion.

The FedRF-TCA approach proposes to perform FDA in two steps:

\paragraph{Local Domain Alignment via RF-TCA}
Here, the source and target model are trained respectively by minimizing the MMD-loss $L_T = L_{\rm MMD}$ as in \eqref{eq:def_mmd_loss} on the target client and the following hybrid loss on the $i$th source client:
\begin{equation}\label{eq:source_loss}
    L_S^{(i)} = L_C^{(i)} + \lambda L_{\rm MMD}^{(i)},
\end{equation}
where $L_C^{(i)}$ denotes the classification loss computed on (a min-batch of) $(\X_S^{(i)}, \Y_S^{(i)})$, $L_{\rm MMD}^{(i)}$ the MMD-loss in \eqref{eq:def_mmd_loss}, and $\lambda > 0$ some trade-off hyperparameter.

Notably, during the training process, the target client disseminates $\bSigma_T \bell_T \in \RR^{2N}$ to \emph{all} source clients.
In contrast, each individual source client transmits \emph{solely} $\bSigma_S^{(i)} \bell_S^{(i)} \in \RR^{2N}$ to the target client, and \emph{no} messages are exchanged between source clients themselves.

Also, the Gaussian random matrix $\bOmega$ to compute the random Fourier features in \Cref{def:RFF} is generated by some predefined random seed $\mathfrak S$ shared by \emph{all} source and target clients, so that no additional communication and/or storage is needed.

This local training procedure is summarized for source and target models in \Cref{alg:souce_client} and \Cref{alg:target_client}, respectively.

\paragraph{Global Parameter Aggregation via FedAvg}

The collection of linear adaptive layers $\W_{\RF}$ and/or the source classifiers $\C$, originating from a potentially unordered set of source clients and the target client, is aggregated using the federated averaging (FedAvg) protocol~\cite{mcmahan2017communication}. 
Precisely, for a given around number $t$, we randomly draw from the (discrete) uniform distribution ${\rm Unif}\{0,K\}$ to determine the cardinality $|\mathcal{S}_t|$ of the set $\mathcal{S}_t$, then to get $\mathcal S_t$ by randomly sampling without replacement $|\mathcal{S}_t|$ elements from the index set $\{1, \ldots, K\}$.

This random sampling procedure emulates potential client and/or message drops during training, which can occur due to factors such as limited network reliability. 
It plays a crucial role in enhancing the robustness of the proposed FedRF-TCA approach.
The resulting global parameter aggregation protocol is described in~\Cref{alg:aggregating}. 

\begin{algorithm}[htb]  
    \caption{Local training of the $i$th source model}\label{alg:souce_client}    
    \begin{algorithmic}[1]
    \STATE {\bfseries Input:} Training set $ (\X_S^{(i)} , \Y_S^{(i)} ) $, local feature extractor $\G_S^{(i)}$ and classifier $\C_S^{(i)}$ at the $i$th source domain $\mathcal D_{S}^{(i)}$, (summarized) target message $\bSigma_T \bell_T$, random index set $\mathcal S_t$, some predefined random seed $\mathfrak S$, and number of random features $N$.

    \STATE {\bfseries Output:} Source message $\bSigma_S^{(i)} \bell_S^{(i)}$, updated linear layer $\W_{\RF_S}^{(i)}$ and classifier $\C_S^{(i)}$.

    \STATE\hspace{1em}Sample a mini-batch from  $(\X_S^{(i)}, \Y_S^{(i)})$. 

    \STATE\hspace{1em}Compute source message $\bSigma_S^{(i)} \bell_S^{(i)}$ as in \Cref{def:RFF} with random seed $\mathfrak S$ and number of random features $N$.
    
    \STATE\hspace{1em} \algorithmicif \ {$i \in \mathcal{S}_t$}\  \algorithmicthen  
      \STATE\hspace{2em} Update $\G_S^{(i)}, \W_{\RF_S}^{(i)}, \C_S^{(i)} $ by  minimizing $L_S^{(i)}$ in \eqref{eq:source_loss} via back-propagation.
    \STATE\hspace{1em} \algorithmicelse
      \STATE\hspace{2em} Update $\G_S^{(i)}, \W_{\RF_S}^{(i)}, \C_S^{(i)} $ only by  minimizing $L_C^{(i)}$ in \eqref{eq:source_loss} via back-propagation.
    \STATE\hspace{1em} \algorithmicend\ \algorithmicif
    
    \STATE {\bfseries return} $\bSigma_S^{(i)} \bell_S^{(i)}$ ,$\W_{\RF_S}^{(i)}$ and $ \C_S^{(i)} $.
    \end{algorithmic}
\end{algorithm}
\begin{algorithm}[htb]  
    \caption{Local training of the target model}\label{alg:target_client}    
    \begin{algorithmic}[1]
    \STATE {\bfseries Input:} Unlabelled data $ \X_T $ and local feature extractor $\G_T$ at the target domain $\mathcal{D}_T$, (summarized) source messages $\{ \bSigma_S^{(i)} \bell_S^{(i)} \}_{i \in \mathcal S_t}$, some predefined random seed $\mathfrak S$, and number of random features $N$.

    \STATE {\bfseries Output:} Target feature $\bSigma_T \bell_T$ and linear layer $\W_{\RF_T}$.

    \STATE\hspace{1em}Sample a mini-batch from $ \X_T $.

    \STATE\hspace{1em}Compute target message $\bSigma_T \bell_T$ as \Cref{def:RFF} with random seed $\mathfrak S$ and number of random features $N$.
    
    \STATE\hspace{1em}Update $\G_T, \W_{\RF_T}$ to minimize the MMD loss in \eqref{eq:def_mmd_loss}.

    \STATE {\bfseries return} $\bSigma_T \bell_T$ and $\W_{\RF_T}$.
    \end{algorithmic}
\end{algorithm}

\begin{algorithm}[htb]
  \caption{Global parameter aggregation}\label{alg:aggregating}    
  \begin{algorithmic}[1]
  \STATE {\bfseries Input:} round number $t$, time interval for classifier aggregation $T_C$, source parameters $\{ \W_{\RF_S}^{(i)} \}_{i \in \mathcal S_t}$ and $\{ \C_S^{(i)} \}_{i \in \mathcal S_t}$, target parameters $\W_{\RF_T}$.
  \STATE {\bfseries Output:} Averaged weight $\W_{\RF}$ or averaged classifier $\C$.

  \STATE\hspace{1em} Compute the aggregated weight $\W_{\RF}$ as the average of source weights $\{ \W_{\RF_S}^{(i)} \}_{i \in \mathcal S_t}$ in the set $\mathcal S_t$ and target weight $\W_{\RF_T}$.
  \STATE\hspace{1em} \algorithmicif \ {$t \hspace{0.1cm} \% \hspace{0.1cm} T_C = 0$}\  \algorithmicthen 
  \STATE\hspace{2em} Compute the aggregated $\C$ as the average of the source classifiers $\{ \C_S^{(i)} \}_{i \in \mathcal S_t}$ in the set $\mathcal S_t$.
  \STATE\hspace{1em} \algorithmicend\ \algorithmicif
  \STATE {\bfseries return} aggregated $\W_{\RF}$ and $\C$.
  \end{algorithmic}
\end{algorithm}

To further alleviate the communication overhead associated with global parameter aggregation in FedRF-TCA, we introduce a modification outlined in \Cref{alg:aggregating}. 
Under this scheme, the classifier $\C$ is aggregated \emph{only once} every $T_C \gg 1$ time intervals, while the weight $\W_{\RF}$ is exchanged in each round of communication. 
This adjustment results in a substantial reduction in communication overhead during training, without any significant compromise in performance.

The whole FedRF-TCA training procedure is summarized in \Cref{alg:FedRF-TCA} and illustrated in \Cref{fig:multi_source_target_system}.

\begin{algorithm}[htb]
    \caption{Federated RF-TCA (FedRF-TCA) }\label{alg:FedRF-TCA}    
    \begin{algorithmic}[1]
    \STATE Model initialization.
    \STATE Generate a random seed $\mathfrak S$ and send to all clients.
    \STATE Determine $T_C$ the time interval to aggregate the classifiers.
    \FOR{each round $t=1,2,\ldots$}
        \STATE Sample a random subset $\mathcal S_t$ from the source index set $\{ 1, \ldots, K \}$ as in \Cref{subsec:sec_arch_net}. 
        \STATE \textbf{Local Domain Alignment}:       
        \STATE\hspace{1em}\parbox[t]{\dimexpr\linewidth-\algorithmicindent}{Compute $\{\bSigma_S^{(i)} \bell_S^{(i)}, \W_{\RF_S}^{(i)}, \C_S^{(i)}\}_{i=1}^K$ at source clients  as in~\Cref{alg:souce_client}.}
        \STATE\hspace{1em}Compute $\W_{\RF_T}, \bSigma_T \bell_T$ at the target client as in~\Cref{alg:target_client}.
        \STATE \textbf{Global Parameter Aggregation}:        
        \STATE\hspace{1em} For each client $i \in \mathcal S_t$, aggregate
        $\W_{\RF_T}$ and $ \W_{\RF_S}^{(i)}$ to get $\W_{\RF}$ as in~\Cref{alg:aggregating}.
        \STATE\hspace{1em} \algorithmicif \ {$t \hspace{0.1cm} \% \hspace{0.1cm} T_C = 0$}\  \algorithmicthen 
        \STATE\hspace{2em} For each client $i \in \mathcal S_t$, aggregate $\C_S^{(i)}$ to get $\C$ as in~\Cref{alg:aggregating}.
        \STATE\hspace{1em} \algorithmicend\ \algorithmicif
        \STATE \textbf{Update Client Models}: 
        \STATE\hspace{1em} Assign $\W_{\RF}$ to both $\W_{\RF_T}$ and $\{ \W_{\RF_S}^{(i)} \}_{i \in \mathcal S_t}$, 
        \STATE\hspace{1em} \algorithmicif \ {$t \hspace{0.1cm} \% \hspace{0.1cm} T_C = 0$}\  \algorithmicthen 
        \STATE\hspace{2em} Assign $\C$ to $\C_T$ and $\{ \C_S^{(i)} \}_{i \in \mathcal S_t}$.
        \STATE\hspace{1em} \algorithmicend\ \algorithmicif
    \ENDFOR
    \STATE {\bfseries return} target classifier $\C_T$.
    \end{algorithmic}
\end{algorithm}

\subsection{Advantages of FedRF-TCA}
\label{subsec:advanteges}

In this section, we discuss the advantages of FedRF-TCA in terms of its communication overhead, robustness, and additional privacy guarantee.
Precisely, when compared to popular federated DA methods such as FADA~\cite{peng2019federated}, FedKA~\cite{sun2023feature}, and FDA~\cite{kang2022communicational}, the communication complexity of FedRF-TCA is \emph{independent} of the sample size $n$ (but depends \emph{only} on $N$ the dimension of random features that can be significantly smaller than $n$, see \Cref{theo:perf_RF_TCA}), and thus (up to) a factor of $n$ times smaller than existing federated DA approaches. 
In addition to the extremely low communication overhead, FedRF-TCA also offers robustness (to both clients dropouts and messages drops due to, e.g., unstable network connection) and additional privacy preserving, while achieving comparably good performance to state-of-the-art federated DA methods.

\Cref{tab:comm_pro} compares the proposed FedRF-TCA protocol to popular federated DA methods, in terms of their communication complexity, robustness (flagged by whether asynchronous training is supported), and additional privacy guarantee\footnote{By ``additional privacy guarantee'' we mean that the federated DA protocol adopts privacy preserving technique in addition to the (innate) data isolation of FL.  For example, the Paillier homomorphic encryption method~\cite{paillier1999public} is used in the training of FDA~\cite{kang2022communicational} as additional privacy protection. }.

\medskip

In the remainder of this section, we analyze the communication complexity and robustness, as well as additional privacy guarantee of FedRF-TCA, in \Cref{subsubsec:communication_overhead}~and~\Cref{subsubsec:privacy}, respectively.

\subsubsection{Analysis of communication overhead and robustness}
\label{subsubsec:communication_overhead}

The communication overhead of the proposed FedRF-TCA protocol depends on the dimension $N$ of the random features (which, according to \Cref{theo:perf_RF_TCA}, suffices to be of order $\log(n)$ and thus depends on the sample size $n$, but \emph{only} in a logarithmic fashion), the dimension of the projected feature space $m$, and the number of the source clients $K$. 
During training, each pair of source and target client will need to exchange the compressed features $\bSigma \bell \in \RR^{2N}$ to compute the MMD loss in \eqref{eq:def_mmd_loss}, the communication overhead of which is $O(K N)$. 
Then, the server collects an aggregation of the linear layer weights $\W_{\RF} \in \mathbb{R}^{2N \times m}$ from all source and target clients.
This leads to a communication overhead of $O(K N m)$. 
As such, the total communication complexity of the proposed FedRF-TCA is $O(KN + K N m)$, and is, in particular, \emph{independent} of the sample size $n$. 

\begin{table}[!htbp]
    \caption{Comparison between different federated DA methods,
    with $K$ the number of clients, $n$ the sample size, $N$ the dimension of features (i.e., the number of random features in FedRF-TCA, the dimension of feature generator in FADA~\cite{peng2019federated}, of feature representation in FedKA~\cite{sun2023feature}, and the latent feature length $L$ in FDA~\cite{kang2022communicational}), 
    and $P \geq 1$ is the ciphertext size of Paillier encryption used in FDA. }
    \centering
    \label{tab:comm_pro}
    \resizebox{\linewidth}{!}{ 
    \begin{tabular}{ccccc}
    \toprule
    Federated DA methods                             &  \makecell{Communication\\ overhead}    & \makecell{Asynchronous\\ training}  & \makecell{Added privacy} \\ \midrule
    FADA~\cite{peng2019federated}       &   $O(K n N)$                   &      \XSolidBrush  &   \XSolidBrush     \\
    FedKA~\cite{sun2023feature}         &   $O(K n N)$                  &       \XSolidBrush &   \XSolidBrush       \\
    FDA~\cite{kang2022communicational}  &   $O(K n N P)$           &     \XSolidBrush    &   \Checkmark      \\
    FedRF-TCA (ours) &   $O(K N)$           &     \Checkmark    &   \Checkmark     \\
    \bottomrule
    \end{tabular}
    }
    
\end{table}

In contrast, other federated DA protocols in \Cref{tab:comm_pro} show a much higher communication overhead, with a complexity growing (at least) linearly in both the feature dimension $N$ \emph{and} the sample size $n$. 
As a concrete example, the FDA protocol~\cite{kang2022communicational} proposes to exchange the encrypted source features of size $O(MP + n_S N P)$ (with encryption cost $M P$ and encrypted source features cost $n_S N P$, $P \geq 1$ the ciphertext size of Paillier encryption), as well as target features of size $n_T N$~\cite{kang2022communicational}, yielding an overall communication overhead of $O(n N P)$, for a single source-target pair.

The proposed FedRF-TCA protocol offers not only a \emph{sample-size-independent} communication overhead, but also robustness to client dropouts and/or messages drops, an important feature for cross-device FL.
In practical FL scenarios, clients may be temporarily unavailable, dropping out or joining during the training procedure, and the sent messages may be lost due to the limited reliability of the network. 
We show with extensive experiments in \Cref{sec:exper} below that the proposed FedRF-TCA is robust to:
\begin{enumerate}
    \item \textbf{clients dropout}, see FedRF-TCA model (\uppercase\expandafter{\romannumeral2}) in Tables~\ref{Office_10_Fedrated}~and~\ref{DigitFive}, where only a random subset of the clients are involved in FDA training; and
    \item \textbf{less frequent aggregation}, see FedRF-TCA model (\uppercase\expandafter{\romannumeral3}) in Tables~\ref{Office_10_Fedrated}~and~\ref{DigitFive}, where classifier is aggregated only every $T_C$ time intervals as in \Cref{alg:aggregating}; and
    \item \textbf{partial dropout of messages}, e.g., randomly dropping out the linear $\W_{\RF_S}$ or classifier $\C_S$ in~\Cref{DigitFive_Ablation}.
\end{enumerate}

\subsubsection{Additional privacy guarantee }
\label{subsubsec:privacy}

Note from \Cref{fig:multi_source_target_system} and the discussion above in \Cref{alg:souce_client}--\ref{alg:RF-TCA} that with the proposed FedRF-TCA protocol, the messages exchanged between different clients are either in form of ``compressed'' random features $\bSigma \bell \in \RR^{2N}$ of linear layer weights $\W_{\RF} \in \RR^{2N \times m}$ and/or the classifier $\C$.
In the following, we discuss the additional privacy guarantee offered by FedRF-TCA.

\begin{Remark}[Privacy protection via random features]\label{rem:feature_protection}\normalfont
FL and federated DA methods are originally proposed to perform decentralized ML so as to mitigate many of the systemic privacy risks~\cite{kairouz2021Advances}.
In addition to the decentralized storage of data at different local clients, many other techniques such as secure aggregation, noise addition, and update clipping have been proposed for additional privacy consideration, against, e.g., a malicious server.
We argue that the proposed FedRF-TCA protocol offers additional privacy protection on top of that of FL and FDA, in the following sense:
\begin{itemize}
  \item[(i)] it is \emph{impossible} for a malicious server to infer the sample size $n_{S,T}$, nor the (random) features $\bSigma_{S,T}$ involved in training from a mere observation of the messages exchanged between clients during FDA training, since $\bSigma_{S,T} \bell_{S,T} = \pm\bSigma_{S,T} \one_{n_{S,T}}/ n_{S,T} \in \RR^{2N}$, the dimension of which is \emph{independent} of the sample size $n_{S,T}$; and 
  \item[(ii)] it is \emph{impossible} to reveal the raw data $\X$ from a mere observation of $\bSigma$, since $\bSigma = \sigma(\bOmega \X)$ by \Cref{def:RFF} for \emph{periodic} function $\sigma(t) = \cos(t)$ or $\sin(t)$ and Gaussian random matrix $\bOmega$, and the solution to infer $\X$ is \emph{not} unique, unless additional constraints are imposed, see also a formal argument in~\cite{zong2023Privacy}.
\end{itemize}
\end{Remark}

\section{Experiments}
\label{sec:exper}

In this section, we provide extensive numerical results on the proposed RF-TCA and FedRF-TCA approaches on various datasets, showing their advantageous performance as well as computational/communicational efficiency and robustness. 
In~\Cref{subsec:exp_detail}, we discuss experimental details on datasets and baselines. 
The computational efficiency and robustness of FedRF-TCA are testified in~Sections~\ref{subsec:efficient}~and~\ref{exp:robustness}, respectively. 
In~\Cref{exp:FedRF_TCA}, we provide additional experimental results demonstrating the advantageous performance of FedRF-TCA with respect to SoTA federated DA methods.
In~\Cref{exp:adlation}, we provide ablation experiments to show the effectiveness of FedRF-TCA.
The optimal performance is shown in \textbf{boldface}, and the second optimal in \underline{underlined}.
Code to reproduce the results in this section are publicly available at \url{https://github.com/SadAngelF/FedRF-TCA}.

\begin{figure*}[!htbp]
  \centering
  \begin{subfigure}[c]{0.43\linewidth}
  \centering
  \captionsetup{skip=0.1pt}
  \begin{tikzpicture}[font=\footnotesize]
    \pgfplotsset{every major grid/.style={style=densely dashed}}
    \begin{axis}[
      height=.55\linewidth,
      width=\linewidth,
      xmin=-1,
      xmax=3.2,
      ymin=0.65,
      ymax=1.0,
      xtick={-1,0,1,2,3},
      xticklabels = {$10^{-1}$,$10^0$, $10^1$, $10^2$,$10^3$},
      ytick={0.65,0.7,0.8,0.9,1.0},
      grid=major,
      scaled ticks=true,
      xlabel={ Running time (s) },
      ylabel={ Accuracy },
      xlabel style={label distance=0.1pt},
      ylabel style={label distance=0.1pt},
      legend style = {at={(0.02,0.98)}, anchor=north west, font=\footnotesize}
      ]
      \addplot[BLUE,mark=o,only marks,line width=1pt] coordinates{
      (-0.4296, 0.6967)(-0.2909, 0.7718)(0.3716, 0.7857)(1.224,0.7925)
      };
      \addplot[mark=*,only marks,line width=1pt] coordinates{(1.090, 0.7495)} node[]{};
      \addplot[RED,mark=*,only marks,line width=1pt] coordinates{(1.090, 0.7495)} node[anchor=west]{TCA};
      \addplot[PURPLE,mark=*,only marks,line width=1pt] coordinates{(2.416, 0.7252)} node[anchor=south]{JDA};
      \addplot[GREEN,mark=*,only marks,line width=1pt] coordinates{(1.949, 0.8429)} node[anchor=east]{CORAL};
      \addplot[BROWN,mark=*,only marks,line width=1pt] coordinates{(2.187, 0.8454)} node[anchor=south]{GFK};
      \addplot[ORANGE,mark=*,only marks,line width=1pt] coordinates{(2.983, 0.8324)} node[anchor=north east]{DaNN};

      \addplot[BLUE,solid,line width=1pt] coordinates{(-0.2909, 0.7718)(0.3716, 0.7857)} node[above=0.1cm]{RF-TCA};
       \addplot[BLUE,solid,line width=1pt] coordinates{(0.3716, 0.7857)(1.224,0.7925)} ;
      \addplot[BLUE,solid,line width=1pt] coordinates{(-0.4296, 0.6967)(-0.2909, 0.7718)};
       \addlegendentry{{ RF-TCA }};
      \addlegendentry{{ Baseline methods }};
    \end{axis}
  \end{tikzpicture}
  \caption{ Office-Caltech } \label{fig:office10}
  \end{subfigure}
  \begin{subfigure}[c]{0.43\linewidth}
  \centering
  \captionsetup{skip=0.1pt}
  \begin{tikzpicture}[font=\footnotesize]
    \pgfplotsset{every major grid/.style={style=densely dashed}}
    \begin{axis}[
      height=.55\linewidth,
      width=\linewidth,
      xmin=-1,
      xmax=4,
      ymin=0.5,
      ymax=0.7,
      xtick={-1,0,1,2,3,4},
      xticklabels = {$10^{-1}$,$10^0$, $10^1$, $10^2$,$10^3$,$10^4$},
      ytick={0.5,0.6,0.7},
      grid=major,
      scaled ticks=true,
      xlabel={ Running time (s) },
      ylabel={ Accuracy },
      xlabel style={label distance=0.1pt},
      ylabel style={label distance=0.1pt},
      legend style = {at={(0.02,0.98)}, anchor=north west, font=\footnotesize}
      ]
      \addplot[BLUE,mark=o,only marks,line width=1pt] coordinates{(-0.5106, 0.5377)(-0.1148, 0.5836)(0.4836, 0.5861)(1.272,0.5892)(2.344, 0.5928)};
      
      \addplot[mark=*,only marks,line width=1pt] coordinates{(2.208, 0.6342)} node[]{};
      \addplot[RED,mark=*,only marks,line width=1pt] coordinates{(2.208, 0.6342)} node[anchor=south west]{TCA};
      \addplot[PURPLE,mark=*,only marks,line width=1pt] coordinates{(3.247, 0.6134)} node[anchor=west]{JDA};
      \addplot[GREEN,mark=*,only marks,line width=1pt] coordinates{(1.964, 0.5924)} node[anchor=south]{CORAL};
      \addplot[BROWN,mark=*,only marks,line width=1pt] coordinates{(2.100, 0.5900)} node[anchor=north]{GFK};
      \addplot[ORANGE,mark=*,only marks,line width=1pt] coordinates{(3.766, 0.5847)} node[anchor=east]{DaNN};
      \addplot[BLUE,solid,line width=1pt] coordinates{(-0.1148, 0.5836)(0.4836, 0.5861)} node[below=0.1cm]{RF-TCA};
      \addplot[BLUE,solid,line width=1pt] coordinates{(0.4836, 0.5861)(1.272,0.5892)};
      \addplot[BLUE,solid,line width=1pt] coordinates{(1.272,0.5892)(2.344, 0.5928)};
      \addplot[BLUE,solid,line width=1pt] coordinates{(-0.5106, 0.5377)(-0.1148, 0.5836)};
      \addlegendentry {{ RF-TCA }};
      \addlegendentry{{ Baseline methods }};
    \end{axis}
  \end{tikzpicture}
  \caption{ Office-31 } \label{fig:office31}
  \end{subfigure}
  
  \caption{
    Classification accuracy and running time of the proposed RF-TCA versus baseline DA methods on DeCAF6 features of Office-Caltech and Office-31 datasets. 
    {\BLUE \bf Blue} circles for RF-TCA approach with a different number of random features $ N \in \{ 100, 500, 1\,000, 2\,000, 5\,000 \}$, the {\RED \bf red}, {\PURPLE \bf purple}, {\GREEN \bf green}, {\BROWN \bf brown} and {\ORANGE \bf orange} for {\RED \bf TCA}~\cite{pan2010domain}, {\PURPLE \bf JDA}~\cite{long2013transfer}, {\GREEN \bf CORAL}~\cite{sun2016return}, {\BROWN \bf GFK}~\cite{gong2012geodesic}, and {\ORANGE \bf DaNN}~\cite{ghifary2014domain} approach, respectively. 
    The results are obtained by averaging over \emph{all} source-target domain pairs ($12$ for Office-Caltech and $6$ for Office-31), see \Cref{app_sec:exp} in the appendix for a detailed exposition of these results. 
  }
  \label{fig:rf-tca}
\end{figure*}
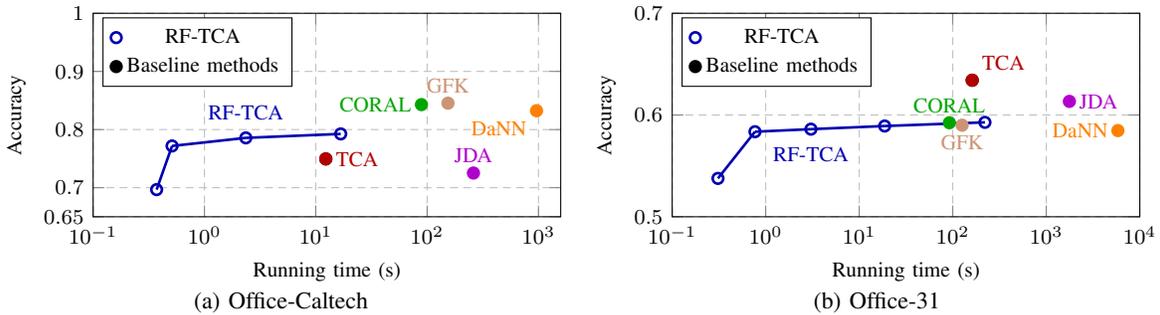

\subsection{Experimental details}
\label{subsec:exp_detail}

Here, we present the datasets and baseline methods to be compared with throughout this section.
\subsubsection{Datasets}
\begin{itemize} 
    \item \textbf{Office-Caltech~\cite{gong2012geodesic}:} This dataset is made up by $10$ common classes shared by Office-31 and Caltech-256 datasets. It contains four domains: Amazon (A), Webcam (W), DLSR (D), and Caltech (C).
    \item \textbf{Office-31~\cite{saenko2010adapting}:} This dataset is a standard benchmark for domain adaptation, with $31$ classes from three different domains: Amazon (A), Webcam (W), and DLSR~(D).
    \item \textbf{Digit-Five~\cite{ganin2016domain}:} This dataset consists handwritten digits from five domains: MNIST~(mn), MNIST-M~(mm), SVHN~(sv), USPS~(up), and Synthetic Digits~(sy) dataset. 
    \item \textbf{Visda-C~\cite{peng2017visda}:} This is a challenging and large-scale dataset with $152\,396$ synthetic 3D model renderings as source and $55\,388$ real-world images as target. 
\end{itemize}

\subsubsection{Baselines}

In terms of computational efficiency, the proposed RF-TCA approach is compared against popular DA methods such as TCA~\cite{pan2010domain}, JDA~\cite{long2013transfer}, CORAL~\cite{sun2016return}, GFK~\cite{gong2012geodesic}, and DaNN~\cite{ghifary2014domain}.
See \Cref{subsec:review} above for a brief review of these methods.
In a federated DA context, the proposed FedRF-TCA scheme is compared against FADA~\cite{peng2019federated} and FADE~\cite{hong2021federated}. 
The baseline results in \Cref{Office_10_Fedrated}, \Cref{DigitFive}, and \Cref{Visda_C} are repeated from \cite{peng2019federated,hong2021federated}.

\subsection{Computational efficiency of RF-TCA}
\label{subsec:efficient}

We compare, in \Cref{fig:rf-tca}, the running time and classification accuracy of the proposed RF-TCA approach against popular DA methods, on both Office-Caltech~\cite{gong2012geodesic} and Office-31~\cite{saenko2010adapting} datasets. 
The ``transferred'' features obtained are classified using a fully-connected neural network with two hidden layers with $100$ neurons per layer.

Analysis of \Cref{fig:rf-tca} clearly demonstrates that our proposed RF-TCA approach achieves the most favorable trade-off between computational complexity and performance when compared to other widely recognized DA methods. 
It consistently exhibits at least a \emph{tenfold} reduction in running time while delivering equally commendable performance.

Notably, our empirical observations indicate that RF-TCA performs exceptionally well with a relatively modest number of random features $N$, aligning with the theoretical insights presented in \Cref{theo:perf_RF_TCA}. 

\subsection{Communication efficiency and robustness of FedRF-TCA}
\label{exp:robustness}

Here, we conduct experiments to demonstrate the communication efficiency and robustness of the proposed FedRF-TCA approach in  a federated DA context over unreliable networks.

Recall that during the training of FedRF-TCA, messages containing data-dependent features $\bSigma_S^{(i)} \bell_S^{(i)}$ and model parameters $ \W_{\RF_S}^{(j)}, \C_S^{(k)}$ are exchanged among participating clients. 
As such, the size of communicated messages of FedRF-TCA is significantly smaller than other FL or FDA algorithms that exchange the whole model (e.g., FedAvg).
\Cref{tab:transfer number} compares the size of communicated messages in FDA training under FedRF-TCA and FedAvg.
We observe that the proposed FedRF-TCA benefits from a significantly less communication overhead (as shown in \Cref{tab:transfer number} and \Cref{tab:comm_pro} above), and at the same time, better FDA performance (as shown in \Cref{fig:FedRF-TCA_frequency} and the results in \Cref{exp:FedRF_TCA}).

\begin{table}
  \centering
  \caption{The size of communicated messages (in million) in each update of FedAvg and FedRF-TCA (with $N = 1\,000$).}
  \label{tab:transfer number}
  \begin{tabular}{cccc}
  \toprule
                      & Data-dependent features & Model parameters    & Sum \\ 
  \midrule                    
   FedAvg             &     0               &  25.637  & 25.637       \\
   FedRF-TCA          &     0.001          &  1.691   &  1.692      \\ 
   \bottomrule
  \end{tabular}
  
  \end{table}

Moreover, in cases where the network experiences instability, messages and/or clients may undergo random dropouts during training. 
Such dropouts could potentially result in a significant deterioration in the model's performance.
In~\Cref{DigitFive_Ablation}, we provide empirical evidence that reassuringly demonstrates this is \emph{not} the case for FedRF-TCA. 
Remarkably, the proposed FedRF-TCA exhibits excellent robustness to network reliability, maintaining its performance even under challenging network conditions.
Precisely, 
\begin{enumerate}
    \item in setting (\uppercase\expandafter{\romannumeral1}), for a given random subset $\mathcal A$, all parameters of $\bSigma_S^{(i)} \bell_S^{(i)}$, $ \W_{\RF_S}^{(j)} $ and $\C_S^{(k)}$ are exchanged; 
    \item in setting (\uppercase\expandafter{\romannumeral2}) a (random subset) of classifiers $ \C_S^{(k)}$ are not involved in training; and
    \item in setting (\uppercase\expandafter{\romannumeral3}) both $ \W_{\RF_S}^{(j)} $ and $\C_S^{(k)}$ are randomly dropped and only $\bSigma_S^{(i)} \bell_S^{(i)}$ are fully exchanged.
\end{enumerate}
We particularly observe from settings (\uppercase\expandafter{\romannumeral2}) and (\uppercase\expandafter{\romannumeral3}) versus (\uppercase\expandafter{\romannumeral1}) that FedRF-TCA is robust against (random) message drops of $ \W_{\RF_S}^{(j)} $ and $ \C_S^{(k)}$, and demonstrates similar performance despite different types of message drops.
This robustness is also consistently observed across different TL tasks.

\begin{table}[htbp]
    \centering
    \caption{ Classification accuracy ($\%$) of FedRF-TCA approach over unreliable networks, on Digit-Five dataset~\cite{ganin2016domain} with communication interval $T_C = 50$. 
    $\mathcal A$ represents a random subset of the (source) index set $\{ 1, \ldots, K \}$ (and corresponds to $\mathcal S_t$ in \Cref{alg:FedRF-TCA}), $\mathcal B$ is a random subset of $\mathcal A$, and $\mathcal C$ is a random subset of $\mathcal B$ obtained as in \Cref{subsec:sec_arch_net}. 
    Settings (\uppercase\expandafter{\romannumeral2}) and (\uppercase\expandafter{\romannumeral3}) model the cases where some messages are dropped due to unstable network connection. 
    For example, in setting (\uppercase\expandafter{\romannumeral3}) only $\bSigma_S^{(i)} \bell_S^{(i)} \,$/$\, \W_{\RF_S}^{(j)} \,$/$\, \C_S^{(k)}$, $i \in \mathcal A, j \in \mathcal B, k \in \mathcal C$ are exchanged during the training of FedRF-TCA. }   
    \label{DigitFive_Ablation}
    \begin{tabular}{ccccccc}
    \toprule
    Task/Setting  & (\uppercase\expandafter{\romannumeral1}) $\mathcal A$ / $\mathcal A$ / $\mathcal A$ & (\uppercase\expandafter{\romannumeral2}) $\mathcal A$ / $\mathcal A$ / $\mathcal B$ & (\uppercase\expandafter{\romannumeral3}) $\mathcal A$ / $\mathcal B$ / $\mathcal C$ \\
    \midrule
    mm,sv,sy,up$\to$mt & \textbf{97.60}  & \underline{97.44} & 97.43 \\
    mt,sv,sy,up$\to$mm & \textbf{65.16}  & 64.03 & \underline{64.34} \\
    mt,mm,sv,sy$\to$up & 89.99 & \underline{90.19} & \textbf{90.45}\\
    \bottomrule
    \end{tabular}
\end{table}
\begin{figure}[htbp]
  \centering
  \begin{tikzpicture}
  \renewcommand{\axisdefaulttryminticks}{4} 
  \pgfplotsset{every major grid/.style={densely dashed}}       
  \tikzstyle{every axis y label}+=[yshift=-10pt] 
  \tikzstyle{every axis x label}+=[yshift=5pt]
  \pgfplotsset{every axis legend/.append style={cells={anchor=east},fill=white, at={(0.02,0.98)}, anchor=north west, font=\tiny}}
  \begin{axis}[
    xmin=2.19,
    xmax=6.8,
    ymin = -.35,
    width=.8\linewidth,
    height=.65\linewidth,
    ytick={0,0.5,1},
    xtick={2.5257, 3.21887, 3.912, 4.6052, 5.2983, 5.9915, 6.6846},
    xticklabels = {$12$, $25$, $50$, $100$,$200$,$400$,$800$},
    grid=major,
    scaled ticks=true,
    xlabel={ Communication interval of classifiers $T_{C}$},
    ylabel={ Accuracy },
    xlabel style={label distance=0.1pt},
    ylabel style={label distance=0.1pt},
    legend style = {at={(0.02,0.02)}, anchor=south west, font=\footnotesize}
    ]


  \addplot[
  mark=+,color=BLUE!60!white,line width=1pt,
  error bars/.cd,
  y dir=both, y explicit,
  error bar style={color=gray}
  ] 
  coordinates{
  (2.304, 0.9763) +- (0,0.0018)
  (2.996, 0.9763) +- (0,0.0007)
  (3.91, 0.9744) +- (0,0.0023)
  (4.605, 0.9742) +- (0,0.0032)
  (5.298, 0.9757) +- (0,0.0016)
  (5.991, 0.9734) +- (0,0.0019)
  (6.685, 0.9632) +- (0,0.0284)
  };
  \addlegendentry{FedRF-TCA: mm,sv,sy,up→mt}

  \addplot[
  dashed, mark=+,color=BLUE!60!white,line width=1pt,
  error bars/.cd,
  y dir=both, y explicit,
  error bar style={color=gray}
  ] 
  coordinates{
  (2.304, 0.5925) +- (0,0.0210)
  (2.996, 0.5922) +- (0,0.0135)
  (3.91, 0.5715) +- (0,0.0123)
  (4.605, 0.5890)  +- (0,0.0204)
  (5.298, 0.5791) +- (0,0.0186)
  (5.991, 0.5459) +- (0,0.0168)
  (6.685, 0.1051) +- (0,0.0220)
  };
  \addlegendentry{FedAvg: mm,sv,sy,up→mt}

  \addplot[
  mark=+,color=GREEN!60!white,line width=1pt,
  error bars/.cd,
  y dir=both, y explicit,
  error bar style={color=gray}
  ] 
  coordinates{
  (2.304, 0.8995) +- (0,0.0113)
  (4.605, 0.8984) +- (0,0.0086)
  (5.298, 0.8999) +- (0,0.0098)
  (5.991, 0.8848) +- (0,0.0351)
  (6.685, 0.8941) +- (0,0.0085)
  };
  \addlegendentry{FedRF-TCA: mm,mt,sv,sy→up}

  \addplot[
  dashed, mark=+,color=GREEN!60!white,line width=1pt,
  error bars/.cd,
  y dir=both, y explicit,
  error bar style={color=gray}
  ] 
  coordinates{
  (2.304, 0.4281) +- (0,0.0213)
  (2.996, 0.4171) +- (0,0.0207)
  (3.91, 0.3987)  +- (0,0.0223)
  (4.605, 0.3761) +- (0,0.0186)
  (5.298, 0.3483) +- (0,0.0198)
  (5.991, 0.2975) +- (0,0.0351)
  (6.685, 0.2601) +- (0,0.0085)
  };
  \addlegendentry{FedAvg: mm,mt,sv,sy→up}

  \end{axis}
  \end{tikzpicture}
  \caption{{Classification accuracy (mean $\pm$ standard deviation) of FedRF-TCA and FedAvg with different communication intervals $T_{C} \in \{ 10, 20, 50, 100, 200, 400, 800 \}$, with in total $1\,600$ rounds of communication, as in (\uppercase\expandafter{\romannumeral1}) of \Cref{DigitFive_Ablation}. }}
  \label{fig:FedRF-TCA_frequency}
\end{figure}

We further show in~\Cref{fig:FedRF-TCA_frequency} that this robustness to network reliability is consistent over different choices of the communication interval $T_C$ in~\Cref{alg:FedRF-TCA}.
By increasing the communication interval $T_C$ for aggregation, the performance of FedRF-TCA remains very stable (with a fluctuation less than $2\%$), while that of FedAvg is much more significant (at least $10\%$). 
Moreover, the FDA performance of FedRF-TCA is consistently better than that of FedAvg.

\subsection{Additional experiments of FedRF-TCA}
\label{exp:FedRF_TCA}

Here, we provide additional numerical results on Office-31~\cite{saenko2010adapting}, Office-Caltech~\cite{gong2012geodesic}, Digit-Five~\cite{ganin2016domain}, and Visda-C~\cite{peng2017visda}, showing the advantageous performance of the proposed FedRF-TCA approach.

\begin{table}[h!]
  \centering
  \caption{Classification accuracy ($\%$) of different DA methods on Office-31 data~\cite{saenko2010adapting} with single source domain. FedRF-TCA is applied on raw data while others on DeCAF6 features~\cite{donahue2014decaf}. } 
  \label{fig:PerfomanceOffice31}
  \begin{tabularx}{\linewidth}{c*{6}{X}c}
  \toprule
  Methods & 
  \resizebox{2.5em}{!}{A$\rightarrow$D} & 
  \resizebox{2.5em}{!}{A$\rightarrow$W} & 
  \resizebox{2.5em}{!}{D$\rightarrow$A} & 
  \resizebox{2.5em}{!}{D$\rightarrow$W} & 
  \resizebox{2.5em}{!}{W$\rightarrow$A} & 
  \resizebox{2.5em}{!}{W$\rightarrow$D} & 
  Avg \\
  \midrule
  TCA~\cite{pan2010domain}                       & \underline{60.24} & 48.55 & \underline{40.75} & \textbf{93.08} & \underline{38.90} & \underline{98.99} &  \underline{63.42} \\
  JDA~\cite{long2013transfer}                    & 57.42 & 45.53 & 39.43 & 91.06 & 37.06 & 97.59 & 61.34 \\
  CORAL~\cite{sun2016return}                     & 54.81 & 49.93 & 35.88 & 89.81 & 33.65 & 91.36 & 59.24 \\
  GFK~\cite{gong2012geodesic}                    & 55.82 & \underline{52.41} & 34.63 & 87.04 & 33.95 & 90.16 & 59.00 \\
  DaNN~\cite{ghifary2014domain}                  & 49.80 & 50.63 & 35.76 & 87.25 & 36.76 & 90.60 & 58.47 \\
  \midrule
  \makecell{RF-TCA \\ ($N = 500$)}          & 55.62 & 51.06 & 29.42 & 87.92 & 32.58 & 93.57 & 58.36 \\
  \makecell{RF-TCA \\ ($N = 1\,000$)}       & 56.02 & 51.57 & 31.30 & 87.67 & 30.13 & 94.97 & 58.61 \\
  \makecell{FedRF-TCA \\ ($N = 1000$)}             & \textbf{80.46} & \textbf{76.56} & \textbf{56.25} & \underline{92.96} & \textbf{59.37} & \textbf{99.21} & \textbf{77.46} \\ 
  \bottomrule
  \end{tabularx}
\end{table}

We compare, in \Cref{fig:PerfomanceOffice31}, the classification accuracy of different (not necessarily federated) DA methods on the Office-31 dataset~\cite{saenko2010adapting}. 
We observe that:
\begin{enumerate}
    \item the performance of the proposed RF-TCA approach matches popular DA baselines, with a significance reduction in both complexity and storage; and
    \item the FedRF-TCA approach, when applied on raw data and on a non-federated setting, also establishes remarkably good performance.
\end{enumerate}

\begin{table}[htb]
  \centering
  \caption{Classification accuracy (\%) on Office-Caltech dataset~\cite{gong2012geodesic} with different federated DA methods. 
  Baseline results are repeated from \cite{peng2019federated}. 
  Setting (\uppercase\expandafter{\romannumeral1}): all clients average both $\W_{\RF}$ and $\C$ in each communication round; (\uppercase\expandafter{\romannumeral2}): only a random subset $\mathcal{S}_t$ of source clients are involved in training; (\uppercase\expandafter{\romannumeral3}): as for (\uppercase\expandafter{\romannumeral2}) with classifier aggregation interval $T_C = 50$. }
  \label{Office_10_Fedrated}
  \setlength{\tabcolsep}{1.6pt}{
  \begin{tabular}{cccccc} 
  \toprule
  Methods                       & C,D,W$\to$A & A,D,W$\to$C & A,C,W$\to$D & A,C,D$\to$W & Avg \\
  \midrule
  ResNet101~\cite{he2016deep}                         & 81.9  & 87.9  & 85.7   & 86.9  & 85.6    \\
  AdaBN~\cite{li2016revisiting}     & 82.2   & 88.2   & 85.9   & 87.4   & 85.7    \\
  AutoDIAL~\cite{maria2017autodial} & 83.3   & 87.7   & 85.6   & 87.1   & 85.9    \\
  f-DAN\footnotemark[3]~\cite{long2015learning,peng2019federated}     & 82.7   & 88.1   & 86.5   & 86.5   & 85.9    \\
  f-DANN\footnotemark[4]~\cite{ganin2015unsupervised,peng2019federated} & 83.5   & 88.5   & 85.9   & 87.1   & 86.3    \\
  FADA~\cite{peng2019federated}     & 84.2   & 88.7   & 87.1   & 88.1   & 87.1   \\
  \midrule
  FedRF-TCA (\uppercase\expandafter{\romannumeral1})        &94.1       &98.1   &\underline{98.9}       &88.9   & \underline{95.0}    \\ 
  FedRF-TCA (\uppercase\expandafter{\romannumeral2})        &\textbf{94.6}      &\textbf{98.7}  &\textbf{99.1}  &\underline{89.6}       & \textbf{95.5}
   \\
  FedRF-TCA (\uppercase\expandafter{\romannumeral3})        &\underline{94.5}   &\underline{98.6}       & 98.8  &\textbf{90.0}  & \textbf{95.5}
   \\
  \bottomrule
  \end{tabular}}
\end{table}
\begin{table}[htb]
  \centering
  \caption{Classification accuracy (\%) on Digit-Five dataset~\cite{ganin2016domain} with different federated DA methods. Baseline results repeated from \cite{peng2019federated}. ``$\to$ mt'' means ``mm,sv,sy,up$\to$mt.''  Settings (\uppercase\expandafter{\romannumeral1}), (\uppercase\expandafter{\romannumeral2}), and (\uppercase\expandafter{\romannumeral3}) as in \Cref{Office_10_Fedrated}. }
  \label{DigitFive}
  \begin{tabular}{ccccccc}
  \toprule
  Methods &  $\to$mt &  $\to$mm &  $\to$up &  $\to$sv &  $\to$sy & Avg  \\ \midrule
  Source Only                                         & 75.4    & 49.6     & 75.5       & 22.7      & 44.3        & 53.5 \\
  f-DANN\footnotemark[4]~\cite{ganin2015unsupervised}                                              & 86.1     & 59.5     & 89.7       & 44.3      & 53.4        & 66.6 \\
  f-DAN\footnotemark[3]~\cite{long2015learning}                                               & 86.4     & 57.5     & \underline{90.8}     & \underline{45.3}     & \underline{58.4}       & \underline{67.7} \\
  FADA~\cite{peng2019federated}                                                & 91.4       & 62.5    & \textbf{91.7}     & \textbf{50.5}      & \textbf{71.8}      & \textbf{73.6} \\ \midrule   
  FedRF-TCA (\uppercase\expandafter{\romannumeral1})  & 97.3    & \underline{64.6} & 89.2       & 41.5      & 43.9      & $67.3$    \\
  FedRF-TCA (\uppercase\expandafter{\romannumeral2})  & \textbf{97.5}   & \textbf{65.3}  & 89.2         & 41.9  & 43.8  & $67.5$        \\        
  FedRF-TCA (\uppercase\expandafter{\romannumeral3})  & \underline{97.4}        & 64.3  & 89.5  & 41.9  & 44.4  & $67.5$         \\
  \bottomrule
  \end{tabular}
\end{table}
\footnotetext[3]{Here, f-DAN is a federated DA method~\cite{peng2019federated} based on DAN~\cite{long2015learning}.}
\footnotetext[4]{Here, f-DANN is a federated DA method~\cite{peng2019federated} based on DANN~\cite{ganin2015unsupervised}.}

In~\Cref{Office_10_Fedrated}~and~\ref{DigitFive}, we present a series of experiments highlighting the exceptional performance of FedRF-TCA in comparison to state-of-the-art federated DA methods. 
Furthermore, we emphasize its robustness in scenarios involving message and/or client dropouts due to poor network conditions.

Notably, our observations in settings (\uppercase\expandafter{\romannumeral2}) and (\uppercase\expandafter{\romannumeral3}) in relation to (\uppercase\expandafter{\romannumeral1}) in \Cref{Office_10_Fedrated}~and~\cref{DigitFive} reveal that exchanging information of only a random subset $\mathcal S_t$ of the source clients and asynchronously aggregating the classifier does \emph{not} adversely affect the performance of FedRF-TCA.

It is also worth highlighting that FedRF-TCA attains the optimal performance on the Office-Caltech dataset (\Cref{Office_10_Fedrated}), and consistently delivers comparable performance on Digit-Five (\Cref{DigitFive}) and Visda-C datasets (\Cref{Visda_C}), all while significantly reducing communication complexity.

\begin{table}[!htb]
    \centering
    \caption{ Classification accuracy (\%) on Visda-C dataset~\cite{peng2017visda}. Baseline results using the FADE~\cite{hong2021federated} strategy.}
    \label{Visda_C}
    \resizebox{\linewidth}{!}{
    \begin{tabular}{ccccc}
    \toprule
    Methods & Source only & DANN~\cite{ganin2015unsupervised}  & CDAN~\cite{long2018conditional} & FedRF-TCA \\
    \midrule
    Centralized & 46.6        & 57.6  & \textbf{73.9}        & \underline{64.5}  \\
    Federated    & 54.3   & 56.4  & \textbf{73.1}        & \underline{62.5}  \\     
    \bottomrule 
    \end{tabular}}
\end{table}

In \Cref{app_sec:exp} of the appendix, we provide further numerical experiments showing the advantageous performance and robustness of the proposed FedRF-TCA approach with classifier aggregation strategies different than FedAvg.

\subsection{Ablation Experiments}
\label{exp:adlation}

In this section, we provide ablation experiments to show the effective design of FedRF-TCA. 
To evaluate the effectiveness of FedRF-TCA, we perform ablation studies of each component of FedRF-TCA, on Digit-Five and Office-Caltech dataset in \Cref{tab:ablation_DigitFive}~and~\Cref{tab:ablation_officecaltech}, respectively.
The performance of FedRF-TCA is consistently better than other ablation settings updated using FedAvg.

\begin{table}[htb]
  \centering
  \caption{{Classification accuracy (\%) on Digit-Five dataset, for ResNet updated using FedAvg, ResNet + RF-TCA Module but updated using FedAvg, and FedRF-TCA.}}
  \label{tab:ablation_DigitFive}
  \resizebox{\linewidth}{!}{
  \begin{tabular}{lcccccc}
  \toprule
  Methods &  $\to$mt &  $\to$mm &  $\to$up &  $\to$sv &  $\to$sy & Avg  \\ \midrule
  ResNet  & \underline{54.2}       & \underline{26.5}    & \underline{48.7}     & \underline{22.3}      & \underline{20.6}      & \underline{34.5} \\ 
  ResNet + RF-TCA Module & 46.8 & 23.9 & 45.0 & 21.4 & 19.0 & 31.2 \\ 
  \midrule   
  FedRF-TCA   & \textbf{97.3}    & \textbf{64.6} & \textbf{89.2}       & \textbf{41.5}      & \textbf{43.9}      & \textbf{67.3}    \\
  \bottomrule
  \end{tabular}
  }
\end{table}

\begin{table}[htb]
  \centering
  \caption{{Classification accuracy (\%) on Office-Caltech dataset, for ResNet updated using FedAvg, ResNet + RF-TCA Module but updated using FedAvg, and  FedRF-TCA.}}
  \label{tab:ablation_officecaltech}
  \resizebox{\linewidth}{!}{
  \begin{tabular}{lccccc} 
  \toprule
  Methods                       & C,D,W$\to$A & A,D,W$\to$C & A,C,W$\to$D & A,C,D$\to$W & Avg \\
  \midrule
  ResNet                 & \underline{91.8} & 84.4    & 95.7     & 94.5     & 91.6      \\ 
  ResNet + RF-TCA Module & 90.4 & \underline{84.9} & \underline{95.9} & \underline{95.3} & \underline{91.6}  \\ 
  \midrule
  FedRF-TCA       & \textbf{94.1}       &\textbf{98.1}   &\textbf{98.9}       &\textbf{88.9}   & \textbf{95.0}    \\ 
  \bottomrule
  \end{tabular}}
\end{table}

\usetikzlibrary {patterns,patterns.meta}
\begin{figure}
  \centering
  \begin{tikzpicture}
  \begin{axis}
  [ybar , 
  grid=major,major grid style={dashed}, 
  ymin=0,  
  ylabel=Test Accuracy (\%),  
  bar width=.5cm, 
  width=14cm,
  height=5cm,  
  symbolic x coords={MNIST-M, Synthetic Digits},  
  x=2.5cm, 
  xtick=data, 
  nodes near coords, 
  nodes near coords style={font=\fontsize{8}{12}\selectfont}, 
  enlarge x limits=0.5, 
  legend style={at={(1.22,1.1)},anchor=north, font=\scriptsize}, 
  legend cell align=right,
  ] 
  \addplot[fill={rgb,255:red,0; green,129; blue,207},
            postaction ={pattern=north east lines}
        ] 
        coordinates {(MNIST-M, 65.3) (Synthetic Digits, 44.4)}; 
  \addplot[fill={rgb,255:red,124; green,187; blue,105},
            postaction ={pattern=horizontal lines}
        ] coordinates {(MNIST-M, 96.4) (Synthetic Digits, 92.5)}; 
  \addplot[fill={rgb,255:red,179; green,243; blue,158},
            postaction ={pattern=north west lines}
        ] coordinates {(MNIST-M, 46.3) (Synthetic Digits, 40.0)}; 
  \addplot[fill={rgb,255:red,94; green,194; blue,181},
            postaction={pattern=crosshatch}
        ] coordinates {(MNIST-M, 55.1) (Synthetic Digits, 48.2)};
  \legend{FedRF-TCA: explicit, FedRF-TCA: implicit, Without $\bSigma \bell$: implicit, FedAvg: implicit}; 
\end{axis} 
  \end{tikzpicture}
  \caption{{Performance of FedRF-TCA with and without $\bSigma \bell$, as well as of FedAvg in the case of \emph{explicit} and \emph{implicit} data heterogeneity. 
  For explicit data heterogeneity, we use the same setting as in \Cref{tab:ablation_DigitFive}; while for implicit data heterogeneity, we evenly divide the MNIST-M (or Synthetic Digits) of Digit-Five dataset into five subsets, so that each subset contains data from \emph{similar} local data distribution. 
  }}
  \label{fig:data_heterogeneity}
\end{figure}

To testify the performance of FedRF-TCA under \emph{implicit data heterogeneity} where different clients have similar local data distribution, we evenly divide the MNIST-M (or Synthetic Digits) of Digit-Five dataset into five subsets, so that each subset contains data from similar local data distribution. 
Then, one subset of data is used at the target client, and the four remaining subsets at the four source clients. 
This setting is referred to as ``implicit'' data heterogeneity in~\Cref{fig:data_heterogeneity}.
We observe from \Cref{fig:data_heterogeneity} that FedRF-TCA significantly outperforms the classical FedAvg approach, and that the transferred data-dependent features $\bSigma \bell$ play an critical role in such advantageous performance.
In particular, note that the performance under implicit data heterogeneity greatly exceeds that under explicit data heterogeneity: This is due to the fact that the local data domains in the implicit data heterogeneity scenario are more similar than in the explicit data heterogeneity (for which a more substantial effort is needed to perform DA).

\section{Conclusion and Perspectives}

In this paper, we propose FedRF-TCA, a robust and communication-efficient federated DA approach based on the improved random features-based TCA (RF-TCA) approach.
The proposed FedRF-TCA approach has a communication complexity that is (theoretically and practically) nearly \emph{independent} of the sample size, and is robust to messages and/or clients dropouts in the network.
We further provide extensive experiments demonstrating the numerical efficiency and advantageous performance of FedRF-TCA.

By leveraging the block matrix structure inherent in the random feature maps in \Cref{def:RFF}, FedRF-TCA can be readily extended vertical FL~\cite{gu2020federated}.
More generally, the FedRF-TCA strategy holds the potential for broader use in other federated MMD-based DA methods. 
Importantly, it accomplishes this while delivering substantial reductions in communication and computational complexity, all without compromising the performance typically associated with vanilla MMD-based methods.

\section*{Acknowledgments}

Z.~Liao would like to acknowledge the National Natural Science Foundation of China (NSFC-62206101) and the Guangdong Provincial Key Laboratory of Mathematical Foundations for Artificial Intelligence (2023B1212010001) for providing partial support.

R.~C.~Qiu would like to acknowledge the National Natural Science Foundation of China (NSFC-12141107) and the Key Research and Development Program of Guangxi (GuiKe-AB21196034) for providing partial support.

J.~Li's work has been has been partially supported by the National Key R\&D Program of China No.~2020YFB1710900 and No.~2020YFB1806700, NSFC Grants 61932014 and 62232011. 

{\footnotesize
\bibliographystyle{IEEEtran}
\bibliography{IEEEabrv,liao,wang,feng}
}

\vspace{11pt}

\clearpage 

\appendices

\onecolumn

\section{Useful Lemmas}
\label{sec:app_lemmas}

\begin{Lemma}[Sherman–Morrison]\label{lem:sherman-morrison}
For $\A \in \RR^{p \times p}$ invertible and $\uu,\vv \in \RR^p$, $\A + \uu \vv^\T$ is invertible if and only if $1+\vv^\T \A^{-1} \uu \neq 0$ and
\begin{equation*}
  (\A + \uu \vv^\T)^{-1} = \A^{-1} - \frac{\A^{-1} \uu \vv^\T \A^{-1} }{1+\vv^\T \A^{-1} \uu}.
\end{equation*}
\end{Lemma}

\begin{Theorem}[Weyl's inequality, {\cite[Theorem~4.3.1]{horn2012matrix}}]\label{theo:weyl}
Let $\A,\B \in \RR^{p \times p}$ be symmetric matrices and let the respective eigenvalues of $\A$, $\B$ and $\A+\B$ be arranged in non-decreasing order, i.e., $\lambda_1 \leq \lambda_2 \leq \ldots \leq \lambda_{p-1} \leq \lambda_p$. Then, for all $i\in\{1,\ldots,p\}$,
\begin{align*}
  &\lambda_i(\A + \B) \leq \lambda_{i+j} (\A) + \lambda_{p-j}(\B), \quad j=0,1,\ldots,p-i,\\
    &\lambda_{i-j+1} (\A) + \lambda_j(\B) \leq \lambda_i(\A + \B), \quad j=1,\ldots,i.
\end{align*}
In particular, taking $i=1$ in the first equation and $i=p$ in the second equation, together with the fact $\lambda_j(\B) =-\lambda_{p+1-j}(-\B)$ for $j=1,\ldots,p$, implies
\begin{align*}
  \max_{1\leq j\leq p} |\lambda_j(\A) - \lambda_j(\B)| \leq \| \A - \B \|_2.
\end{align*}
\end{Theorem}

\begin{Theorem}[Davis--Kahan theorem,~\cite{yu2015useful}]\label{theo:davis-kahan}
Under the same notation as in Theorem~\ref{theo:weyl}, assume that for a given $i \in \{ 1, \ldots, n\}$ one has $ \min \{ | \lambda_{i-1}(\A ) - \lambda_i(\A) |, | \lambda_{i+1}(\A ) - \lambda_i(\A) | > 0 $, then the corresponding eigenvectors satisfy 
\begin{equation}
  \sin \Theta\left( \uu_i (\A + \B), \uu_i (\A) \right) \leq \frac{ 2 \| \B \|_2 }{ \min \{ | \lambda_{i-1}(\A ) - \lambda_i(\A) |, | \lambda_{i+1}(\A ) - \lambda_i(\A) | \} },
\end{equation}
where we denote `$\sin \Theta$' the alignment between two vectors $\uu_1, \uu_2 \in \RR^{n}$ as
\begin{equation}
  \Theta(\uu_1, \uu_2) \equiv \arccos \left( \frac{\uu_1^\T \uu_2}{\| \uu_1 \|_2 \cdot \| \uu_2 \|_2} \right),
\end{equation}
that satisfies $\| \uu_1 - \uu_2 \| \leq \sqrt 2 \sin \Theta(\uu_1, \uu_2)$.
\end{Theorem}

\section{Facts and Discussions about Vanilla TCA}
\label{sec:appendix_fact_TCA}

Here, we discuss some interesting properties of the vanilla TCA approach in \eqref{eq:vanilla_TCA}.

Recall from \Cref{lem:equivalent_form} that the ``transformed'' features obtained by vanilla TCA is given by the top eigenspace of 
\begin{equation}
    \A \equiv \H \left( \K^2 - \frac{ \K^2 \bell \bell^\T \K^2 }{\gamma + \bell^\T \K^2 \bell} \right) \H,
\end{equation}
as defined in \eqref{eq:def_A}.
As a consequence, we have the following remark on the regularization parameter $\gamma $ of vanilla TCA.

\begin{Remark}[On regularization of vanilla TCA]\label{rem:regu_vanilla_TCA}\normalfont
Per its definition in \eqref{eq:def_ell} and let $\lambda_{\max}(\K) \geq \lambda_{\min} (\K) \geq 0$ denote the maximum and minimum eigenvalue of $\K$, one has $\| \bell \|_2^2 = \frac{n}{n_S n_T}$, and that the (only) non-zero eigenvalue of $\frac{ \K^2 \bell \bell^\T \K^2 }{\gamma + \bell^\T \K^2 \bell}$ is $\frac{\bell^\T \K^4 \bell}{\gamma + \bell^\T \K^2 \bell}$ with
\begin{equation}\label{eq:regu_vanilla_TCA}
  \frac{\lambda_{\min}^4(\K) n }{ \gamma n_S n_T + \lambda_{\min}^2(\K) n } \leq \frac{\bell^\T \K^4 \bell}{\gamma + \bell^\T \K^2 \bell} \leq \frac{\lambda_{\max}^4(\K) n }{ \gamma n_S n_T + \lambda_{\max}^2(\K) n }.
\end{equation}

As such, for $\lambda_{\min}(\K), \lambda_{\max}(\K)$ both of order $\Theta(1)$ with respect to $n$, one has:
\begin{enumerate}
  \item if the number of source and target data are ``balanced'', in the sense that both $n_S, n_T$ are of order $\Theta(n)$, then the rank-one matrix $\frac{ \K^2 \bell \bell^\T \K^2 }{\gamma + \bell^\T \K^2 \bell}$ is of spectral norm order $\Theta(1)$, and thus ``on even ground'' with $\K^2$ (that is, the spectral norm of the rank-one matrix is of the same order as $\K^2$ even for $n$ large), \emph{if and only if} one sets $\gamma = \Theta(n^{-1})$; and
  \item if the number of source and target data are ``unbalanced'' with $n_S =\Theta(1)$ and $n_T = \Theta(n)$, or $n_S = \Theta(n)$ and $n_T = \Theta(1)$ (recall that $n_S + n_T = n$), then $\frac{ \K^2 \bell \bell^\T \K^2 }{\gamma + \bell^\T \K^2 \bell}$ is ``on even ground'' with $\K^2$ in a spectral norm sense \emph{if and only if} $\gamma = \Theta(1)$.
\end{enumerate}
A similar statement can be made for the R-TCA approach introduced in \Cref{subsec:R-TCA}.
\end{Remark}

\begin{proof}[Proof of \Cref{rem:regu_vanilla_TCA}]
Since $\frac{ \K^2 \bell \bell^\T \K^2 }{\gamma + \bell^\T \K^2 \bell} \K^2 \bell = \K^2 \bell \frac{\bell^\T \K^4 \bell }{\gamma + \bell^\T \K^2 \bell} $, so the (only) non-zero eigenvalue of $\frac{ \K^2 \bell \bell^\T \K^2 }{\gamma + \bell^\T \K^2 \bell}$ is $\frac{\bell^\T \K^4 \bell}{\gamma + \bell^\T \K^2 \bell}$.


\begin{equation}\label{eq:part_of_vanilla_TCA}
  \frac{ \bell^\T \K^4 \bell }{\gamma + \bell^\T \K^2 \bell} = \frac{\frac{ \bell^\T \K^4 \bell }{\| \bell \|_2^2}}{\frac{\gamma + \bell^\T \K^2 \bell}{\| \bell \|_2^2}}, \text{ and }
  \frac{\frac{ \bell^\T \K^2 \bell }{\| \bell \|_2^2}}{\frac{\gamma + \bell^\T \K^2 \bell}{\| \bell \|_2^2}} = 1 - \frac{\frac{ \gamma }{\| \bell \|_2^2}}{\frac{\gamma }{\| \bell \|_2^2} + \frac{\bell^\T \K^2 \bell}{\| \bell \|_2^2}} 
\end{equation}

We perform an eigenvalue decomposition $\K = U \Lambda U^T$, where $\Lambda$ is the eigenvalue matrix, 
so $\bell^\T \K^4 \bell = \bell^\T (U \Lambda^4 U^T) \bell$. 
And since $\lambda_{\max}(\K) $ and $ \lambda_{\min} (\K) $ denote the maximum and minimum eigenvalue of $\K$, 
so we get  $ \lambda_{\min}^2(\K) \leq \frac{\bell^\T \K^2 \bell}{\| \bell \|_2^2} \leq \lambda_{\max}^2(\K)$ 
and $ \lambda_{\min}^2(\K) \frac{ \bell^\T \K^2 \bell }{\| \bell \|_2^2} \leq \frac{ \bell^\T \K^4 \bell }{\| \bell \|_2^2} \leq \lambda_{\max}^2(\K) \frac{ \bell^\T \K^2 \bell }{\| \bell \|_2^2} $. 

For the right-hand side term in \Cref{eq:part_of_vanilla_TCA}, we have the following expression:

\begin{equation}
  1 - \frac{\frac{ \gamma }{\| \bell \|_2^2}}{\frac{\gamma }{\| \bell \|_2^2} + \lambda_{\min}^2(\K) } \leq 1 - \frac{\frac{ \gamma }{\| \bell \|_2^2}}{\frac{\gamma }{\| \bell \|_2^2} + \frac{\bell^\T \K^2 \bell}{\| \bell \|_2^2}} \leq 1 - \frac{\frac{ \gamma }{\| \bell \|_2^2}}{\frac{\gamma }{\| \bell \|_2^2} + \lambda_{\max}^2(\K) } 
\end{equation}

So, we get 
\begin{align}
  \frac{ \bell^\T \K^4 \bell }{\gamma + \bell^\T \K^2 \bell} & = \frac{\frac{ \bell^\T \K^4 \bell }{\| \bell \|_2^2}}{\frac{\gamma + \bell^\T \K^2 \bell}{\| \bell \|_2^2}} \leq \frac{ \lambda_{\min}^2(\K) \frac{ \bell^\T \K^2 \bell }{\| \bell \|_2^2}}{\frac{\gamma + \bell^\T \K^2 \bell}{\| \bell \|_2^2}}
  = \lambda_{max}^2(\K) \left( 1 - \frac{\frac{ \gamma }{\| \bell \|_2^2}}{\frac{\gamma }{\| \bell \|_2^2} + \frac{\bell^\T \K^2 \bell}{\| \bell \|_2^2}} \right) \\
  & \leq \lambda_{max}^2(\K) \left( 1 - \frac{\frac{ \gamma }{\| \bell \|_2^2}}{\frac{\gamma }{\| \bell \|_2^2} + \lambda_{\max}^2(\K) } \right) 
  = \frac{\lambda_{\max}^4(\K) n }{ \gamma n_S n_T + \lambda_{\max}^2(\K) n }
\end{align}

\begin{align}
  \frac{ \bell^\T \K^4 \bell }{\gamma + \bell^\T \K^2 \bell} \geq \lambda_{min}^2(\K) \left( 1 - \frac{\frac{ \gamma }{\| \bell \|_2^2}}{\frac{\gamma }{\| \bell \|_2^2} + \frac{\bell^\T \K^2 \bell}{\| \bell \|_2^2}} \right) \geq \lambda_{min}^2(\K) \left( 1 - \frac{\frac{ \gamma }{\| \bell \|_2^2}}{\frac{\gamma }{\| \bell \|_2^2} + \lambda_{\min}^2(\K) } \right) = \frac{\lambda_{\min}^4(\K) n }{ \gamma n_S n_T + \lambda_{\min}^2(\K) n }
\end{align}

Thus we concludes the proof of \Cref{rem:regu_vanilla_TCA}
\end{proof}

Figure~\ref{fig:results_rankone_Amazon_Webcam} provides numerical evidence for Remark~\ref{rem:regu_vanilla_TCA}, by showing the performance of vanilla TCA as a function of the regularization parameter $\gamma$, for {\sf Amazon} to {\sf Webcam} from the Office-31 dataset, with $n_S = 2817$ and $n_T = 795$. 
In accordance with \Cref{rem:regu_vanilla_TCA}, the classification accuracy of vanilla TCA varies as a function of regularization $\gamma$ \emph{only} when $\gamma$ belongs to a specific interval (marked in gray in Figure~\ref{fig:Vanilla_TCA__Amazon_Webcam}). 
A similar behavior can be observed for R-TCA in \Cref{fig:R_TCA_Amazon_Webcam}.

Remark~\ref{rem:regu_vanilla_TCA} is of direct algorithmic use in the search of optimal regularization parameter $\gamma$ for vanilla TCA, R-TCA, and consequently RF-TCA.


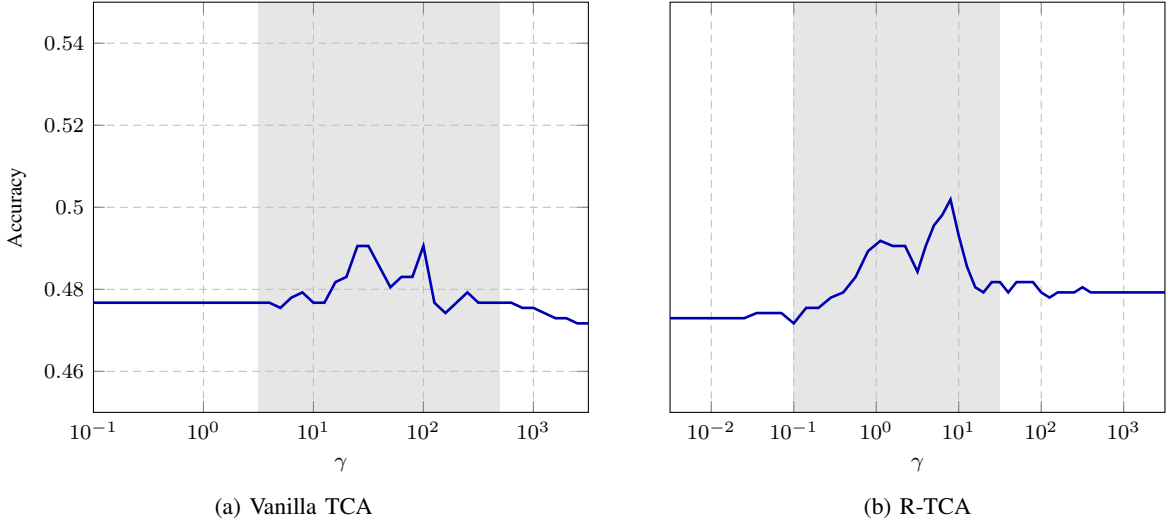
\begin{figure}%
\centering
\subfloat[Vanilla TCA]{
\centering
\begin{tikzpicture}[font=\footnotesize]
    \pgfplotsset{every major grid/.style={style=densely dashed}}
    \begin{axis}[
      width=.45\linewidth,
      xmin=-1,
      xmax=3.5,
      ymin=0.45,
      ymax=0.55,
      xtick={-2,-1,0,1,2,3},
      xticklabels = {$10^{-2}$,$10^{-1}$,$10^0$,$10^1$,$10^2$,$10^3$},
      grid=major,
      scaled ticks=true,
      xlabel={ $ \gamma $ },
      ylabel={ Accuracy },
      legend style = {at={(0.98,0.98)}, anchor=north east, font=\footnotesize}
      ]
      \addplot[BLUE,line width=1pt] coordinates{
      (-1.0,0.4767295597484277)(-0.9,0.4767295597484277)(-0.8,0.4767295597484277)(-0.7,0.4767295597484277)(-0.6,0.4767295597484277)(-0.5,0.47672956)(-0.4,0.47672956)(-0.3,0.47672956)(-0.2,0.47672956)(-0.1,0.47672956)(0,0.47672956)(0.1,0.47672956)(0.2,0.47672956)(0.3,0.47672956)(0.4,0.47672956)(0.5,0.47672956)(0.5,0.47672956)(0.6,0.47672956)(0.7,0.475471698)(0.8,0.477987421)(0.9,0.479245283)(1,0.47672956)(1.1,0.47672956)(1.2,0.481761006)(1.3,0.483018868)(1.4,0.490566038)(1.5,0.490566038)(1.5,0.490566038)(1.6,0.485534591)(1.7,0.480503145)(1.8,0.483018868)(1.9,0.483018868)(2,0.490566038)(2.1,0.47672956)(2.2,0.474213836)(2.3,0.47672956)(2.4,0.479245283)(2.5,0.47672956)(2.5,0.47672956)(2.6,0.47672956)(2.7,0.47672956)(2.8,0.47672956)(2.9,0.475471698)(3,0.475471698)(3.1,0.474213836)(3.2,0.472955975)(3.3,0.472955975)(3.4,0.471698113)(3.5,0.471698113)
      } ;
    \begin{pgfonlayer}{background}
    \fill [color=black!10!white] (axis cs:0.5,0.45) rectangle (axis cs:2.7,0.55);
    \end{pgfonlayer}
    \begin{pgfonlayer}{background}
    \end{pgfonlayer}
    \end{axis}
  \end{tikzpicture}
  \label{fig:Vanilla_TCA__Amazon_Webcam}
}
\qquad
\subfloat[R-TCA]{
\centering
\begin{tikzpicture}[font=\footnotesize]
    \pgfplotsset{every major grid/.style={style=densely dashed}}
    \begin{axis}[
      width=.45\linewidth,
      xmin=-2.5,
      xmax=3.5,
      ymin=0.45,
      ymax=0.55,
      xtick={-2,-1,0,1,2,3},
      xticklabels = {$10^{-2}$,$10^{-1}$,$10^0$,$10^1$,$10^2$,$10^3$},
      ytick=\empty,
      grid=major,
      scaled ticks=true,
      xlabel={ $ \gamma $ },
      ylabel={ },
      legend style = {at={(0.98,0.98)}, anchor=north east, font=\footnotesize}
      ]
      \addplot[BLUE,line width=1pt] coordinates{
      (-2.5,0.4729559748427673)(-2.35,0.4729559748427673)(-2.2,0.4729559748427673)(-2.05,0.4729559748427673)(-1.9,0.4729559748427673)(-1.75,0.4729559748427673)(-1.6,0.4729559748427673)(-1.45,0.4742138364779874)(-1.3,0.4742138364779874)(-1.1500000000000001,0.4742138364779874)(-1.0,0.4716981132075472)(-0.8500000000000001,0.47547169811320755)(-0.7000000000000002,0.47547169811320755)(-0.55,0.4779874213836478)(-0.3999999999999999,0.47924528301886793)(-0.25,0.4830188679245283)(-0.10000000000000006,0.48930817610062893)(0.049999999999999795,0.4918238993710692)(0.19999999999999973,0.49056603773584906)(0.35000000000000003,0.49056603773584906)(0.5,0.48427672955974843)(0.6000000000000001,0.49056603773584906)(0.7000000000000002,0.49559748427672956)(0.8,0.4981132075471698)(0.9000000000000001,0.5018867924528302)(1.0,0.4930817610062893)(1.1,0.48553459119496856)(1.2000000000000002,0.48050314465408805)(1.3,0.47924528301886793)(1.4000000000000001,0.4817610062893082)(1.5,0.4817610062893082)(1.6,0.47924528301886793)(1.7000000000000002,0.4817610062893082)(1.8000000000000003,0.4817610062893082)(1.9000000000000004,0.4817610062893082)(2.0,0.47924528301886793)(2.1,0.4779874213836478)(2.2,0.47924528301886793)(2.3000000000000003,0.47924528301886793)(2.4000000000000004,0.47924528301886793)(2.5,0.48050314465408805)(2.6,0.47924528301886793)(2.7,0.47924528301886793)(2.8000000000000003,0.47924528301886793)(2.9000000000000004,0.47924528301886793)(3.0,0.47924528301886793)(3.1,0.47924528301886793)(3.2,0.47924528301886793)(3.3000000000000003,0.47924528301886793)(3.4000000000000004,0.47924528301886793)(3.5,0.47924528301886793)
      } ;
    \begin{pgfonlayer}{background}
    \fill [color=black!10!white] (axis cs:-1,0.45) rectangle (axis cs:1.5,0.55);
    \end{pgfonlayer}
    \begin{pgfonlayer}{background}
    \end{pgfonlayer}
    \end{axis}
  \end{tikzpicture}
  \label{fig:R_TCA_Amazon_Webcam}
}%
\caption{ Classification accuracy versus the regularization parameter $\gamma$ on DECAF6 features of Office-31 data for $n_S = 2817$ ({\sf Amazon}), $n_T = 795$ ({\sf Webcam}) using a SVM classifier. Here, $\bell^\T \K^2 \bell = 10^{-0.793} = 0.1607$ and $\bell^\T \K \bell = 10^{-1.874} = 0.01336$.}
\label{fig:results_rankone_Amazon_Webcam}
\end{figure}

\section{Proof of \Cref{theo:perf_RF_TCA}}
\label{sec:proof_of_theo_RF_TCA}

In this section, we present the proof of \Cref{theo:perf_RF_TCA}.

Before we go into the detailed proof of \Cref{theo:perf_RF_TCA}, recall first from \Cref{lem:TCA-VS-R-TCA} in \Cref{subsec:R-TCA} that the R-TCA approach defined in \eqref{eq:alternative_TCA} is ``equivalent'' to the vanilla TCA up to a change of variable in the regularization.
As such, it suffices to study the behavior of the R-TCA, the solution $\W_{\rm R}$ of which satisfies
\begin{equation}\label{eq:R_TCA_solution_eigen_sym}
  \A_{\rm R} \cdot \H \K \W_{\rm R} \equiv \frac1{\gamma} \H \left( \K - \frac{\K \bell \bell^\T \K}{ \gamma + \bell^\T \K \bell } \right) \H \cdot \H \K \W_{\rm R} =  \H \K \W_{\rm R} \cdot \bLambda_{\rm R},
\end{equation}
per the Sherman–Morrison lemma (\Cref{lem:sherman-morrison}) for $\gamma + \bell^\T \K \bell \neq 0$.

As for R-TCA, it can be checked that also be shown that the transferred features $\W_{\RF}$ obtained by RF-TCA take a similar as
\begin{equation}\label{eq:def_A_RF}
  \A_{\RF} \cdot \H \bSigma^\T \W_{\RF} \equiv \H \bSigma^\T (\bSigma \L \bSigma^\T + \gamma \I_{2N})^{-1} \bSigma \H \cdot \H \bSigma^\T \W_{\RF} = \H \bSigma^\T \W_{\RF} \cdot \bLambda_{\RF},
\end{equation}
with diagonal $\bLambda_{\RF} \in \RR^{m \times m}$ containing the largest $m$ eigenvalues of $(\bSigma \L \bSigma^\T + \gamma \I_{2N})^{-1} \bSigma \H \bSigma^\T$, and thus of the symmetric matrix
\begin{equation}\label{eq:RF_TCA_1}
  \A_{\RF} = \H \bSigma^\T (\bSigma \L \bSigma^\T + \gamma \I_{2N})^{-1} \bSigma \H = \frac1\gamma \H \left( \bSigma^\T \bSigma - \frac{\bSigma^\T \bSigma \bell \bell^\T \bSigma^\T \bSigma}{\gamma + \bell^\T \bSigma^\T \bSigma \bell} \right) \H,
\end{equation}
again up to row and column centering via $\H$.

Comparing \eqref{eq:R_TCA_solution_eigen_sym} against \eqref{eq:RF_TCA_1}, we see that the RF-TCA formulation takes a similar form to the that of R-TCA, with the RFF Gram matrix $\bSigma^\T \bSigma $ instead of the Gaussian kernel matrix $\K = \K_{\Gauss}$.
Since $\bSigma^\T \bSigma$ well approximates $\K$ in a spectral norm sense for $N$ large, by Theorem~\ref{theo:RFF_approx_spectral}, one may expect that the two matrices $\A_{\rm R}$ and $\A_{\RF}$ are also close in spectral norm.
This is  precisely given in the following result, the proof of which follows from an (almost immediate) application of Theorem~\ref{theo:RFF_approx_spectral}.

\begin{Corollary}\label{coro:RF_TCA_concentration}
For a given data matrix $\X \in \RR^{p \times n}$, denote the associated Gaussian kernel matrix $\K_{\Gauss} = \K \in \RR^{n \times n}$ and the random Fourier features matrix $\bSigma \in \RR^{2 N \times n}$ as in \Cref{def:RFF}. 
Then, for any given $\varepsilon \in (0, 1)$, there exists a universal constant $C >0$ such that if $N \geq C \dim(\K) \log(n)/\varepsilon^2$, the following (expected) relative spectral norm error bound holds
\begin{equation}
  \frac{\EE \left\| \K - \frac{\K \bell \bell^\T \K}{ \gamma + \bell^\T \K \bell } - (\bSigma^\T \bSigma - \frac{\bSigma^\T \bSigma \bell \bell^\T \bSigma^\T \bSigma}{\gamma + \bell^\T \bSigma^\T \bSigma \bell} ) \right\|_2 }{\| \K \|_2}  \leq \varepsilon.
\end{equation}
\end{Corollary}

\begin{proof}[Proof of \Cref{coro:RF_TCA_concentration}]
In the sequel, we will use the shortcut notation $\hat \K \equiv \bSigma^\T \bSigma \in \RR^{n \times n}$. 
We aim to bound the following expected operator norm difference
\begin{equation}
\EE \left\| \K - \frac{\K \bell \bell^\T \K}{ \gamma + \bell^\T \K \bell } - \left(\bSigma^\T \bSigma - \frac{\bSigma^\T \bSigma \bell \bell^\T \bSigma^\T \bSigma}{\gamma + \bell^\T \bSigma^\T \bSigma \bell} \right) \right\|_2 = \EE \left\| \K - \frac{\K \bell \bell^\T \K}{ \gamma + \bell^\T \K \bell } - \left(\hat \K - \frac{\hat \K \bell \bell^\T \hat \K}{\gamma + \bell^\T \hat \K \bell} \right) \right\|_2.
\end{equation}

First, we have, for $\EE\| \K - \hat \K \|_2 \leq \varepsilon$, by $\| \A + \B \|_2 \leq \| \A \|_2 + \| \B \|_2 $, $\| \A \B \|_2 \leq \| \A \|_2 \cdot \| \B \|_2$, and $\gamma >0$ that
\begin{align}\label{eq:norm_coro1}
&\EE \left\| \frac{\hat \K \bell \bell^\T \hat \K}{\gamma + \bell^\T \hat \K \bell } - \frac{\K \bell \bell^\T \K}{\gamma + \bell^\T \K \bell } \right\|_2 
    \leq \EE \left[ \frac{\left\| \gamma(\hat \K \bell \bell^\T \hat \K-\K \bell \bell^\T \K) + \bell^\T \K \bell \hat \K \bell \bell^\T \hat \K - \bell^\T \hat \K \bell \K \bell \bell^\T \K \right\|_2}{(\gamma + \bell^\T \hat \K \bell)(\gamma + \bell^\T \K \bell)} \right] \nonumber \\
    & \leq \EE \left[ \frac{\left\| \gamma(\hat \K \bell \bell^\T \hat \K-\K \bell \bell^\T \K) \right\|_2}{(\gamma + \bell^\T \hat \K \bell)(\gamma + \bell^\T \K \bell)} \right] + \EE \left[ \frac{\left\| \bell^\T \K \bell \hat \K \bell \bell^\T \hat \K - \bell^\T \hat \K \bell \K \bell \bell^\T \K \right\|_2}{(\gamma + \bell^\T \hat \K \bell)(\gamma + \bell^\T \K \bell)} \right].
\end{align}

Then, for the first right-hand side term in \eqref{eq:norm_coro1}, we have
\begin{align}
    \EE \left[\frac{\left\| \gamma(\hat \K \bell \bell^\T \hat \K-\K \bell \bell^\T \K) \right\|_2}{(\gamma + \bell^\T \hat \K \bell)(\gamma + \bell^\T \K \bell)} \right]
    & = \EE \left[\frac{\gamma\left\| (\hat \K- \K)\bell \bell^\T \hat \K + \K \bell \bell^\T \hat \K - \K \bell \bell^\T \K \right\|_2}{(\gamma + \bell^\T \hat \K \bell)(\gamma + \bell^\T \K \bell)} \right] \nonumber \\
    & \leq \gamma \EE \left[\frac{\left\| \hat \K - \K \right\|_2 \cdot \bell^\T \hat \K \bell }{(\gamma + \bell^\T \hat \K \bell)(\gamma + \bell^\T \K \bell)} \right] + \gamma \EE \left[ \frac{ \bell^\T \K \bell \cdot \left\| \hat \K - \K \right\|_2}{(\gamma + \bell^\T \hat \K \bell)(\gamma + \bell^\T \K \bell)} \right] \nonumber \\
    & \leq \frac{\gamma \EE \left\| \hat \K - \K \right\|_2 }{\gamma + \bell^\T \K \bell} + \frac{\gamma \EE \left\| \hat \K - \K \right\|_2 }{\gamma + \bell^\T \hat \K \bell} \leq 2\varepsilon, \label{eq:part1_norm}
\end{align}
where we use the fact that $ \frac{ \bell^\T \Z \bell }{ \gamma + \bell^\T \Z \bell } \leq 1$ for p.s.d.~$\Z = \hat \K, \K$ and $\gamma > 0$ in the third inequality, and that $\frac{\gamma}{ \gamma + \bell^\T \Z \bell } \leq 1$ as well as $\EE \| \K - \hat \K \|_2 \leq \varepsilon$ in the fourth inequality.

For the second right-hand side term in \eqref{eq:norm_coro1}, we have, again for $\EE\| \K - \hat \K \| \leq \varepsilon$ that
\begin{align*}
&\EE \left[ \frac{\left\| \hat \K \bell \bell^\T \K \bell \bell^\T \hat \K -  \K \bell \bell^\T \hat \K \bell \bell^\T \K \right\|_2}{(\gamma + \bell^\T \hat \K \bell)(\gamma + \bell^\T \K \bell)} \right] \\ 
&= \EE \left[ \frac{\left\|  \hat \K \bell \bell^\T \K \bell \bell^\T \hat \K -  \K \bell \bell^\T \K \bell \bell^\T \hat \K + \K \bell \bell^\T \K \bell \bell^\T \hat \K - \K \bell \bell^\T \hat \K \bell \bell^\T \hat \K + \K \bell \bell^\T \hat \K \bell \bell^\T \hat \K - \K \bell \bell^\T \hat \K \bell \bell^\T \K \right\|_2}{(\gamma + \bell^\T \hat \K \bell)(\gamma + \bell^\T \K \bell)} \right]  \\
& \leq \EE \left[ \frac{ \left\| (\hat \K- \K) \bell \bell^\T \K \bell \bell^\T \hat \K \right\|_2}{(\gamma + \bell^\T \hat \K \bell)(\gamma + \bell^\T \K \bell)} \right] + \EE \left[ \frac{\left\|  \K \bell \bell^\T (\K - \hat \K) \bell \bell^\T \hat \K \right\|_2}{(\gamma + \bell^\T \hat \K \bell)(\gamma + \bell^\T \K \bell)} \right] + \EE \left[ \frac{\left\|  \K \bell \bell^\T \hat \K \bell \bell^\T (\hat \K - \K) \right\|_2}{(\gamma + \bell^\T \hat \K \bell)(\gamma + \bell^\T \K \bell)} \right]  \\
&\leq 3 \cdot \EE \left[\frac{ \left\| \hat \K - \K \right\|_2 (\bell^\T \K \bell) (\bell^\T \hat \K \bell) }{(\gamma + \bell^\T \hat \K \bell)(\gamma + \bell^\T \K \bell)} \right] \leq 3 \varepsilon.
\end{align*}

Putting together, this gives, for $\EE \left\| \K - \hat \K \right\|_2 \leq \varepsilon$ that
\begin{equation}
\EE \left[ \left\| \K - \frac{\K \bell \bell^\T \K}{ \gamma + \bell^\T \K \bell } - \left(\bSigma^\T \bSigma - \frac{\bSigma^\T \bSigma \bell \bell^\T \bSigma^\T \bSigma}{\gamma + \bell^\T \bSigma^\T \bSigma \bell} \right) \right\|_2 \right] \leq 5 \varepsilon,
\end{equation}
which, together with Theorem~\ref{theo:RFF_approx_spectral}, concludes the proof of Corollary~\ref{coro:RF_TCA_concentration}.
\end{proof}

With \Cref{coro:RF_TCA_concentration} at hand, we are ready to present the proof of \Cref{theo:perf_RF_TCA} as follows.
\begin{proof}[Proof of \Cref{theo:perf_RF_TCA}]
Note from \eqref{eq:R_TCA_solution_eigen_sym}~and~\eqref{eq:def_A_RF} that the columns of $\H \K \W_{\rm R} \in \RR^{n \times m}$ and $\H \bSigma^\T \W_{\RF} \in \RR^{n \times m}$ are in fact the top $m$ eigenvectors $\uu_i(\A_{\rm R}), \uu_i( \A_{\RF} ) \in \RR^n, i \in \{1, \ldots, m\}$, of $\A_{\rm R} \equiv \H \K (\gamma \K + \K \bell \bell^\T \K )^{-1} \K \H$ and $\A_{\RF} \equiv \H \bSigma^\T (\gamma \I_{2N} + \bSigma \bell \bell^\T \bSigma^\T)^{-1} \bSigma \H $, respectively. 
As such, the Frobenius norm error satisfies
\begin{equation}
    \EE \left[ \| \H ( \bSigma^\T \W_{\RF} - \K \W_{\rm R} ) \|_2^2 \right] = \sum_{i=1}^m \EE \left[ \| \uu_i( \A_{\RF} ) - \uu_i(\A_{\rm R}) \|_2^2 \right].
\end{equation}
It then follows from Davis--Kahan theorem \cite{yu2015useful}, Theorem~\ref{theo:davis-kahan} in Appendix~\ref{sec:app_lemmas}, that
\begin{align}
    \left\| \uu_i( \A_{\RF} ) - \uu_i(\A_{\rm R}) \right\|_2 & \leq \sqrt 2 \sin\Theta \left(\uu_i( \A_{\RF} ), \uu_i(\A_{\rm R})\right) \nonumber \\
    & \leq \frac{2 \sqrt 2\left\| \A_{\RF} - \A_{\rm R} \right\|_2}{\min\{|\lambda_{i-1}(\A_{\rm R})-\lambda_{i}(\A_{\rm R})|,|\lambda_{i+1}(\A_{\rm R})-\lambda_{i}(\A_{\rm R})| \}},
\end{align}
with `$\sin \Theta(\uu_1, \uu_2)$ the ``sine similarity'' between two vectors $\uu_1, \uu_2 \in \RR^{n}$ with $\Theta(\uu_1, \uu_2) \equiv \arccos \left( \frac{\uu_1^\T \uu_2}{\| \uu_1 \|_2 \cdot \| \uu_2 \|_2} \right)$ that satisfies $\| \uu_1 - \uu_2 \|_2 \leq \sqrt 2 \Theta(\uu_1, \uu_2)$.

Further note that for $\EE \left\| \K - \frac{\K \bell \bell^\T \K}{ \gamma + \bell^\T \K \bell } - (\bSigma^\T \bSigma - \frac{\bSigma^\T \bSigma \bell \bell^\T \bSigma^\T \bSigma}{\gamma + \bell^\T \bSigma^\T \bSigma \bell} ) \right\|_2 \leq \varepsilon \| \K \|_2$, one has
\begin{equation}
  \EE \left\| \A_{\RF} - \A_{\rm R} \right\|_2 \leq \| \H \|_2 \cdot \varepsilon \| \K \|_2 \cdot \| \H \|_2 \leq \varepsilon \| \K \|_2,
\end{equation}
where we use the fact that $\| \H \| = 1$, and therefore
\begin{align}
    \EE \left[ \left\| \H ( \bSigma^\T \W_{\RF} - \K \W ) \right\|_2^2 \right] = \sum_{i=1}^m \EE \left[ \left\| \uu_i( \A_{\RF} ) - \uu_i(\A_{\rm R}) \right\|_2^2 \right] \leq \frac{8m\varepsilon^2}{ (\Delta_\lambda/\| \K \|_2)^2} = \frac{8m\varepsilon^2}{ \tilde{\Delta}_\lambda^2},
\end{align}
where we recall that $\Delta_\lambda \equiv \min_{1 \leq i \leq m} \left|\lambda_i(\A_{\rm R})-\lambda_{i+1}(\A_{\rm R}) \right| > 0$ denotes the \emph{eigen-gap} of $\A_{\rm R} \equiv \H \K (\gamma \K + \K \bell \bell^\T \K )^{-1} \K \H$, as well as the \emph{relative eigen-gap}
\begin{equation}
  \tilde{\Delta}_{\lambda} \equiv \Delta/\| \K \|_2,
\end{equation}
as in the statement of \Cref{theo:perf_RF_TCA}.
Applying Corollary~\ref{coro:RF_TCA_concentration}, we concludes the proof of Theorem~\ref{theo:perf_RF_TCA}.
\end{proof}

\section{Additional numerical experiments}
\label{app_sec:exp}

\subsection{Additional numerical experiments on RF-TCA}
\subsubsection{Experiments setting}

In this section, we provide additional experimental results on RF-TCA. 
All experiments are performed on a machine with Intel(R) Core(TM) i7-7700 CPU @ 3.60GHz and 3090Ti GPU.

Below are the details of experiment in this section:
\begin{itemize}
  \item \textbf{Datasets}: The Office-Caltech \cite{gong2012geodesic} dataset has $4$ subsets (Amazon, Caltech, Dslr and Webcam) with $10$ classes in each subset. The Office-31 \cite{saenko2010adapting} dataset has $3$ subsets (Amazon, Dslr and Webcam) with $31$ classes in each subset. Experiments are performed on DECAF6 features~\cite{donahue2014decaf} of both datasets. 
  All datasets are available at \url{https://github.com/SadAngelF/FedRF-TCA}.
  Both source and target data vectors are normalized to have unit Euclidean norms. 
  \item \textbf{Classifiers}: We use $3$ types of classifiers in this section: fully-connected neural network (FCNN), support vector machine (SVM) and k-nearest neighbor (kNN). FCNN is a fully-connected neural network with two hidden layers (having $100$ neurons per layer). SVM uses the (Gaussian) RBF kernel and follows the same hyperparameter searching protocol as below. And the parameter $k$ of kNN is $1$.
  \item \textbf{Hyperparameters}: There are five hyperparameters in RF-TCA: the number of random features $N$, the dimension of common feature space $m$, the regularization parameter $\gamma$, and the Gaussian (width) kernel parameter $\sigma$ in $\K_{\Gauss} = \{ \exp(- \| \x_i - \x_j \|^2/ (2\sigma^2) ) \}_{i,j=1}^n$. 
  We choose $m = 100$, and search $\gamma$ in the set $\{10^{-3}, 10^{-2}, 10^{-1}, 1, 10^{1}, 10^{2}, 10^{3} \}$, the Gaussian (width) parameter $\sigma$ in the set $\{5, 6, \cdots, 14, 15\}$. For each test, we perform hyperparameter search in the range above and report the best performance.
\end{itemize}

\subsubsection{Numerical results on RF-TCA}

\begin{table}[!htbp]
  \centering
  \caption{ Classification accuracy ($\%$) of different DA methods on Office-Caltech dataset with single source domain. Methods are applied on DeCAF6 features~\cite{donahue2014decaf}.}  
  \label{PerfomanceOffice10}
  \begin{tabular}{ccccccccc}
  \hline
  Methods         & TCA    & \makecell{RF-TCA \\ (with $N = 1\,000$)} & \makecell{RF-TCA \\ (with $N = 500$)}  & Vanilla TCA & JDA    & CORAL  & GFK    & DaNN\\ \hline
  A$\rightarrow$C & 81.03 & 76.49 & 76.04 & 80.76 & 81.12 & \textbf{84.32} & 82.57 & 83.92\\ 
  A$\rightarrow$D & 83.43 & 83.43 & 80.25 & \textbf{87.89} & 84.07 & 84.71 & 0.8535 & 83.00\\
  A$\rightarrow$W & 74.57 & 66.10 & 68.13 & 73.55 & 65.76 & 73.89 & \textbf{75.93} & 75.67\\
  C$\rightarrow$A & 87.89 & 91.02 & 91.44 & 87.78 & 87.99 & 92.37 & \textbf{92.42} & 90.30\\
  C$\rightarrow$D & 63.69 & \textbf{90.44} & 89.17 & 58.59 & 57.32 & 88.53 & 89.38 & 90.00\\
  C$\rightarrow$W & 62.71 & 79.32 & 77.28 & 61.01 & 58.64 & 79.32 & \textbf{83.50} & 79.67\\
  D$\rightarrow$A & 82.15 & 74.53 & 69.93 & 81.31 & 79.64 & \textbf{85.07} & 79.91 & 79.60\\
  D$\rightarrow$C & 52.89 & 59.75 & 58.50 & 50.75 & 48.88 & 76.84 & \textbf{78.57} & 73.92\\
  D$\rightarrow$W & 80.00 & 90.84 & 90.50 & 80.00 & 80.00 & 97.96 & \textbf{99.10} & 96.67\\
  W$\rightarrow$A & 77.55 & 72.86 & 69.93 & \textbf{77.97} & 76.82 & 77.66 & 75.21 & 74.00\\
  W$\rightarrow$C & 62.42 & 58.05 & 55.03 & 59.83 & 63.49 & 70.79 & \textbf{72.62} & 68.50\\
  W$\rightarrow$D & 91.08 & \textbf{100.00} & \textbf{100.00} & 91.71 & 86.62 & \textbf{100.00} & \textbf{100.00} & \textbf{100.00}\\
  Avg             & 74.95 & 78.57 & 77.18 & 74.26 & 72.52 & 84.29 & \textbf{84.54} & 83.24\\ \hline
  \end{tabular}
  \end{table}

  \begin{table}[!htbp]
  \centering
   \caption{ Running time (s) of different DA methods on Office-Caltech dataset with single source domain. Methods are applied on DeCAF6 features~\cite{donahue2014decaf}.}
  \label{TimeOffice10}
  \begin{tabular}{ccccccccc}
  \hline
  Methods         & TCA    & \makecell{RF-TCA \\ (with $N = 1\,000$)} & \makecell{RF-TCA \\ (with $N = 500$)}  & Vanilla TCA & JDA    & CORAL  & GFK    & DaNN\\ \hline
  A$\rightarrow$C & 38.76 & 2.812 & \textbf{0.6355} & 3.072  & 604.7 & 92.91 & 117.6 & 983.6\\ 
  A$\rightarrow$D & 5.231 & 2.349 & \textbf{0.4795} & 0.5285 & 180.9 & 90.27 & 191.3 & 973.8\\
  A$\rightarrow$W & 7.981 & 2.438 & \textbf{0.4993} & 0.7088 & 203.3 & 87.09 & 191.5 & 957.5\\
  C$\rightarrow$A & 38.63 & 2.892 & \textbf{0.6337} & 3.067  & 588.1 & 89.51 & 124.9 & 1004\\
  C$\rightarrow$D & 10.92 & 2.569 & \textbf{0.5157} & 0.7563 & 249.0 & 90.99 & 117.9 & 955.6\\
  C$\rightarrow$W & 10.75 & 2.538 & \textbf{0.5471} & 0.9705 & 251.8 & 88.10 & 119.8 & 947.4\\
  D$\rightarrow$A & 5.356 & 2.373 & \textbf{0.4776} & 0.5294 & 183.9 & 89.27 & 185.1 & 924.4\\
  D$\rightarrow$C & 10.99 & 2.475 & \textbf{0.4985} & 0.7536 & 249.1 & 90.44 & 116.3 & 947.4\\
  D$\rightarrow$W & 0.2790 & 0.4151 & 0.4116 & \textbf{0.07043} & 81.59 & 87.74 & 187.6 & 908.7\\
  W$\rightarrow$A & 7.774 & 2.459 & \textbf{0.4993} & 0.7024 & 210.4 & 86.76 & 184.7 & 959.9\\
  W$\rightarrow$C & 10.66 & 2.527 & \textbf{0.5304} & 0.9671 & 245.5 & 87.18 & 123.4 & 965.5\\
  W$\rightarrow$D & 0.2802 & 2.387 & 0.4131 & \textbf{0.07328} & 82.09 & 87.92 & 186.6 & 975.3\\
  Avg             & 12.30 & 2.353 & \textbf{0.5118} & 1.016 & 260.9 & 89.01 & 153.9 & 961.8\\ \hline
  \end{tabular}
  \end{table}

  \begin{table}[!htbp]
  \centering
  \caption{ Classification accuracy ($\%$) of different DA methods on Office-31 dataset with single source domain. Methods are applied on DeCAF6 features~\cite{donahue2014decaf}.}
  \label{PerfomanceOffice31}
  \begin{tabular}{ccccccccc}
  \hline
  Methods         & TCA    & \makecell{RF-TCA \\ (with $N = 1\,000$)} & \makecell{RF-TCA \\ (with $N = 500$)}  & Vanilla TCA & JDA    & CORAL  & GFK    & DaNN\\ \hline
  A$\rightarrow$D & \textbf{60.24} & 56.02 & 55.62 & 59.83 & 57.42 & 54.81 & 55.82 & 49.80\\
  A$\rightarrow$W & 48.55 & 51.57 & 51.06 & 47.29 & 45.53 & 49.93 & \textbf{52.41} & 50.63\\
  D$\rightarrow$A & \textbf{40.75} & 31.30 & 29.42 & 39.61 & 39.43 & 35.88 & 34.63& 35.76\\
  D$\rightarrow$W & \textbf{93.08} & 87.67 & 87.92 & 92.95 & 91.06 & 89.81 & 87.04 & 87.25\\
  W$\rightarrow$A & 38.90 & 30.13 & 32.58 & \textbf{39.33} & 37.06 & 33.65 & 33.95 & 36.76\\
  W$\rightarrow$D & \textbf{98.99} & 94.97 & 93.57 & 97.79 & 97.59 & 91.36 & 90.16 & 90.60\\
  Avg             & \textbf{63.42} & 58.61 & 58.36 & 62.80 & 61.34 & 59.24 & 59.00 & 58.47\\ \hline
  \end{tabular}
  \end{table}
  \begin{table}[!htbp]
  \centering
  \caption{ Running time (s) of different DA methods on Office-31 dataset with single source domain. Methods are applied on DeCAF6 features~\cite{donahue2014decaf}.}
  \label{TimeOffice31}
  \begin{tabular}{ccccccccc}
  \hline
  Methods         & TCA    & \makecell{RF-TCA \\ (with $N = 1\,000$)} & \makecell{RF-TCA \\ (with $N = 500$)}  & Vanilla TCA & JDA    & CORAL  & GFK    & DaNN\\ \hline
  A$\rightarrow$D & 206.2 & 3.257 & \textbf{0.8764} & 13.48 & 2207 & 96.64 & 125.9 & 6078\\
  A$\rightarrow$W & 265.4 & 3.459 & \textbf{0.9498} & 17.28 & 2911 & 93.35 & 133.2 & 6066\\
  D$\rightarrow$A & 208.0 & 3.415 & \textbf{0.8695} & 13.49 & 2211 & 93.06 & 124.4 & 5783\\
  D$\rightarrow$W & 11.16 & 2.365 & \textbf{0.4760} & 0.8867 & 243.8 & 86.94 & 116.5 & 5545\\
  W$\rightarrow$A & 266.9 & 3.428 & \textbf{0.9389} & 17.32 & 2779 & 93.38 & 127.3 & 5941\\
  W$\rightarrow$D & 11.13 & 2.343 & \textbf{0.4953} & 0.8869 & 251.1 & 88.64 & 128.8 & 5620\\
  Avg             & 161.5 & 3.045 & \textbf{0.7676} & 10.55 & 1767 & 92.00 & 126.0 & 5839\\ \hline
  \end{tabular}
  \end{table}

  As demonstrated in \Cref{PerfomanceOffice10}, \Cref{TimeOffice10}, \Cref{PerfomanceOffice31} and \Cref{TimeOffice31}, the RF-TCA method achieves performance comparable to other methods while requiring significantly less runtime. Furthermore, as the dataset size increases, the computational efficiency of the RF-TCA method becomes even more pronounced.
  
  \begin{table}[!htbp]
    \centering
    \caption{Classification accuracy ($\%$) of different kernels of RF-TCA (with $N = 500$) on Office-Caltech dataset with DECAF6 features.}
    \label{performance_kernel_Office10}
    \begin{tabular}{cccccccccccccc}
    \toprule
    Kernel          & A$\rightarrow$C & A$\rightarrow$D & A$\rightarrow$W & C$\rightarrow$A & C$\rightarrow$D & C$\rightarrow$W & D$\rightarrow$A & D$\rightarrow$C & D$\rightarrow$W & W$\rightarrow$A & W$\rightarrow$C & W$\rightarrow$D & Avg \\ \midrule
    Laplace & \textbf{81.57} & 80.25 & \textbf{69.15} & 90.29 & \textbf{89.81} & \textbf{77.29} & \textbf{70.25} & \textbf{59.22} & 75.93 & \textbf{72.55} & \textbf{57.35} & 96.82 & 76.70  \\
    Gauss   & 76.04  & 80.25 & 68.13 & \textbf{91.44} & 89.17 & 77.28 & 69.93 & 58.50 & \textbf{90.50} & 69.93 & 55.03 & \textbf{100.00} & \textbf{77.18} \\ \bottomrule
    \end{tabular}
    \end{table}

\begin{table}[!htbp]
  \centering
  \caption{Classification accuracy ($\%$) of different kernels of RF-TCA (with $N = 500$) on Office-31 dataset with DECAF6 features.}
  \label{performance_kernel_Office31}
  \begin{tabular}{ccccccccc}
  \toprule
  Kernel          & A$\rightarrow$D & A$\rightarrow$W & D$\rightarrow$A & D$\rightarrow$W & W$\rightarrow$A & W$\rightarrow$D & Avg     \\ \midrule
  Laplace & \textbf{59.64} & 50.44 & \textbf{35.78} & 87.80 & \textbf{34.26} & \textbf{93.98} & \textbf{60.32}  \\
  Gauss   & 55.62  & \textbf{51.06}  & 29.42  & \textbf{87.92}  & 32.58  & 93.57  & 58.36 \\ \bottomrule
  \end{tabular}
\end{table}

While our primary focus in this work is on the Gaussian kernel, it is worth noting that the RF-TCA method is applicable to other kernel functions, as demonstrated in both \Cref{performance_kernel_Office10} and \Cref{performance_kernel_Office31}.

\subsection{Additional source classifier adaptation strategies on FedRF-TCA}

\subsubsection{Experiments setting}

In this section, we provide additional experimental results on FedRF-TCA. 

Below are the details of experiment in this section:
\begin{itemize}
  \item \textbf{Datasets}: The Office-Caltech \cite{gong2012geodesic} dataset has $4$ subsets (Amazon, Caltech, Dslr and Webcam) with $10$ classes in each subset. The Digit-Five \cite{ganin2015unsupervised} dataset has $5$ subsets (Mnist, MNIST-M, USPS, SVHN, Synthetic Digits) with $10$ classes in each subset. . Experiments are performed on raw data oof both datasets.
  \item \textbf{Classifiers}: We only use fully-connected neural network (FCNN). FCNN is a fully-connected neural network with two hidden layers (having $100$ neurons per layer). 
  \item \textbf{Hyperparameters}: The FedRF-TCA model shares the same five hyperparameters as the RF-TCA model. 
\end{itemize}

\subsubsection{One-shot performance for FedRF-TCA}

In \Cref{alg:FedRF-TCA}, though clients only need to transmit source classifiers every $T_C$ round, we attempt to figure out whether the performance remains the same if clients only transmit source classifiers once at the end of training. Thus we conduct a ablation experiment as in~\Cref{fig:FedRF-TCA_frequency_one_shot} based on the same setting as~\Cref{fig:FedRF-TCA_frequency} with extra experiment when $T_C = 1600$. The performance when source classifiers are only transmitted once (the total round number is $1600$ and $T_C = 1600$) is very unstable.

\begin{figure}[htbp]
  \centering
  \begin{tikzpicture}
  \renewcommand{\axisdefaulttryminticks}{4} 
  \pgfplotsset{every major grid/.style={densely dashed}}       
  \tikzstyle{every axis y label}+=[yshift=-10pt] 
  \tikzstyle{every axis x label}+=[yshift=5pt]
  \pgfplotsset{every axis legend/.append style={cells={anchor=east},fill=white, at={(0.02,0.98)}, anchor=north west, font=\tiny}}
  \begin{axis}[
    height=.35\linewidth,
    width=.6\linewidth,
    xmin=2.19,
    xmax=7.5,
    ymin=0,
    xtick={2.5257, 3.21887, 3.912, 4.6052, 5.2983, 5.9915, 6.6846, 7.378},
    xticklabels = {$12$, $25$, $50$, $100$,$200$,$400$,$800$,$1600$},
    grid=major,
    scaled ticks=true,
    xlabel={ communication cycles of classifiers ($\rm T_{C}$)},
    ylabel={ Accuracy },
    xlabel style={label distance=0.1pt},
    ylabel style={label distance=0.1pt},
    legend style = {at={(0.965,0.98)}, anchor=north east, font=\footnotesize}
    ]    
  
  \addplot[
  mark=+,color=BLUE!60!white,line width=1pt,
  error bars/.cd,
  y dir=both, y explicit,
  error bar style={color=gray}
  ] 
  coordinates{
  (2.304, 0.9763) +- (0,0.0018)
  (2.996, 0.9763) +- (0,0.0007)
  (3.91, 0.9744) +- (0,0.0023)
  (4.605, 0.9742) +- (0,0.0032)
  (5.298, 0.9757) +- (0,0.0016)
  (5.991, 0.9734) +- (0,0.0019)
  (6.685, 0.9632) +- (0,0.0284)
  (7.36, 0.7356) +- (0,0.3756)
  };
  \addlegendentry{mm,sv,sy,up→mt}

  \addplot[
  mark=+,color=GREEN!60!white,line width=1pt,
  error bars/.cd,
  y dir=both, y explicit,
  error bar style={color=gray}
  ] 
  coordinates{
  (2.304, 0.8995) +- (0,0.0113)
  (2.996, 0.8946) +- (0,0.0079)
  (3.91, 0.8936) +- (0,0.0063)
  (4.605, 0.8984) +- (0,0.0086)
  (5.298, 0.8999) +- (0,0.0098)
  (5.991, 0.8848) +- (0,0.0351)
  (6.685, 0.8941) +- (0,0.0085)
  (7.39, 0.6606) +- (0,0.3637)
  };
  \addlegendentry{mm,mt,sv,sy→up}
  
  \addplot[
  mark=+,color=RED!60!white,line width=1pt,
  error bars/.cd,
  y dir=both, y explicit,
  error bar style={color=gray}
  ] 
  coordinates{
  (2.304, 0.6495) +- (0,0.0166)
  (2.996, 0.6549) +- (0,0.0079)
  (3.91, 0.6516) +- (0,0.0184)
  (4.317, 0.6407) +- (0,0.0137)
  (4.605, 0.6463) +- (0,0.0084)
  (5.298, 0.6287) +- (0,0.0171)
  (5.991, 0.6393) +- (0,0.0147)
  (6.685, 0.6287) +- (0,0.0253)
  (7.33, 0.4261) +- (0,0.2456) 
  };
  \addlegendentry{mt,sv,sy,up→mm}
  \end{axis}
  \end{tikzpicture}
  \caption{Classification accuracy (\%) of FedRF-TCA on Digit-Five dataset~\cite{ganin2016domain} for different communication cycles $\rm T_{C} \in \{ 10, 20, 50, 100, 200, 400, 800, 1600 \}$ , and total round number is $1600$. The setting is the same as (\uppercase\expandafter{\romannumeral1}) in \Cref{DigitFive_Ablation}.
  The red line represents the average accuracy, while the vertical gray lines indicate the standard deviation. In order to enhance clarity, we have slightly spaced out the x-axis values for the experiment at $T_C=1600$. } 
  \label{fig:FedRF-TCA_frequency_one_shot}
\end{figure}
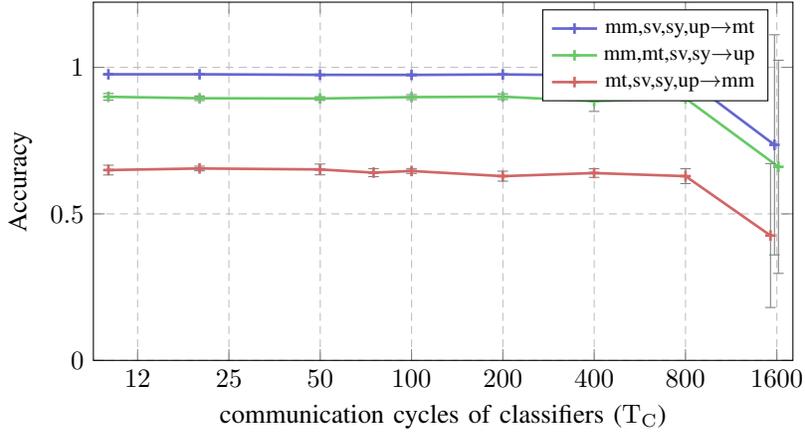

We cancel the source classifiers parameter aggregation step ($T_C \to \infty$) in \Cref{alg:FedRF-TCA} and apply a hard voting strategy utilizing the source classifiers to vote for target predictions. All source clients transmit the classifiers only once at the end of the training process. Results are shown in \Cref{Hard_Voting_Office_10_Fedrated} and \Cref{Hard_Voting_DigitFive}. The results of additional ablation experiments on robustness and asynchrony are shown in \Cref{DigitFive_Asynchrony_Ablation}.

\subsubsection{One-shot hard voting strategy}

We cancel the source classifiers parameter aggregation step ($T_C \to \infty$) in \Cref{alg:FedRF-TCA} and apply a hard voting strategy utilizing the source classifiers to vote for target predictions. All source clients transmit the classifiers only once at the end of the training process. Results are shown in \Cref{Hard_Voting_Office_10_Fedrated} and \Cref{Hard_Voting_DigitFive}. The results of additional ablation experiments on robustness and asynchrony are shown in \Cref{DigitFive_Asynchrony_Ablation}

\begin{table*}[]
  \centering
  \caption{
  One-shot hard voting for source classifiers in FedRF-TCA.  Classification accuracy (\%) on Office-Caltech dataset~\cite{gong2012geodesic} with different federated DA methods. 
  Baseline results are cited from \cite{peng2019federated}.
  Setting (\uppercase\expandafter{\romannumeral1}): all clients average both $\W_{\RF}$ and $\C$ in each communication round; (\uppercase\expandafter{\romannumeral2}): only a random subset $\mathcal{S}_t$ of source clients are involved in training; (\uppercase\expandafter{\romannumeral3}): as for (\uppercase\expandafter{\romannumeral2}) with classifier aggregation interval $T_C = 50$. 
  }
  \label{Hard_Voting_Office_10_Fedrated}
  \begin{tabular}{cccccc}
    \toprule
    FDA methods                       & C, D, W $\to$ A & A, D, W $\to$ C & A, C, W $\to$ D & A, C, D $\to$ W & Average \\
  \midrule
  ResNet101~\cite{he2016deep}                         & 81.9  & 87.9  & 85.7   & 86.9  & 85.6    \\
  AdaBN~\cite{li2016revisiting}     & 82.2   & 88.2   & 85.9   & 87.4   & 85.7    \\
  AutoDIAL~\cite{maria2017autodial} & 83.3   & 87.7   & 85.6   & 87.1   & 85.9    \\
  f-DAN\footnotemark[1]~\cite{long2015learning,peng2019federated}     & 82.7   & 88.1   & 86.5   & 86.5   & 85.9    \\
  f-DANN\footnotemark[2]~\cite{ganin2015unsupervised,peng2019federated} & 83.5   & \underline{88.5}   & 85.9   & 87.1   & 86.3    \\
  FADA~\cite{peng2019federated}     & 84.2   & \textbf{88.7}   & 87.1   & 88.1   & 87.1   \\
  \midrule
  FedRF-TCA (\uppercase\expandafter{\romannumeral4})   & \textbf{93.9}  & 87.6  & \underline{98.4}  & \underline{95.2}  & \textbf{93.8}   \\
  FedRF-TCA (\uppercase\expandafter{\romannumeral5})   & \underline{89.64}  & 87.13  & \textbf{98.88}  & \textbf{98.25}  & \underline{93.48}   \\
  \bottomrule
  \end{tabular}
\end{table*}

\begin{table*}[]
    \centering
    \caption{One-shot hard voting for source classifiers in FedRF-TCA. Classification accuracy (\%) on Digit-Five dataset~\cite{ganin2016domain} with UFDA protocol. Baseline results are cited from \cite{peng2019federated}.}
    \label{Hard_Voting_DigitFive}
    \begin{tabular}{ccccccc}
    \toprule
    FDA methods & mm, sv, sy, up $\to$ mt & mt, sv, sy, up $\to$ mm & mt, mm, sv, sy $\to$ up & mt, mm, sy, up $\to$ sv & mt, mm, sv, up $\to$ sy & Avg  \\ \midrule
    Source Only                                         & 75.4    & 49.6     & 75.5       & 22.7      & 44.3        & 53.5 \\
    f-DANN\footnotemark[2]~\cite{ganin2015unsupervised}                                              & 86.1     & 59.5     & 89.7       & 44.3      & 53.4        & 66.6 \\
    f-DAN\footnotemark[1]~\cite{long2015learning}                                               & 86.4     & \underline{57.5}     & \underline{90.8 }     & \underline{45.3 }     & \underline{58.4 }       & \underline{67.7} \\
    FADA~\cite{peng2019federated}                                                & \underline{91.4}       & \textbf{62.5}    & \textbf{91.7 }     & \textbf{50.5 }      & \textbf{71.8 }      & \textbf{73.6} \\ \midrule   

    FedRF-TCA (\uppercase\expandafter{\romannumeral4})      & 86.5      & 53.7      & 88.0      & 38.3      & 38.1      & 60.8 \\
    FedRF-TCA (\uppercase\expandafter{\romannumeral5})      & \textbf{92.6}      & 55.5      & 89.2      & 36.0      & 38.3      & 62.23 \\        
    \bottomrule
    \end{tabular}
\end{table*}

The aggregation strategy for source classifiers in FedRF-TCA can take various forms. While the hard voting strategy depicted in \Cref{Hard_Voting_Office_10_Fedrated} and \Cref{Hard_Voting_DigitFive} may have implications for source privacy, it continues to yield satisfactory performance. It is possible to explore improved aggregation strategies that balance privacy protection and communication cost more effectively.

Due to the unpredictable nature of the real communication environment, it is challenging for client $i$ to transmit both pairs of messages, $\bSigma^{(i)} \bell^{(i)}$ and $\W_{RF}^{(i)}$, even though these messages in FedRF-TCA are lightweight. So we conduct a ablation experiments on the passing order and the asynchrony of the messages for the one-shot hard voting strategy. As in~\Cref{fig:multi_source_target_system}, we divide the transmission strategy for $\bSigma^{(i)} \bell^{(i)}$ and $\W_{RF}^{(i)}$ into separate categories, which are denoted as ``all'', ``ordered'', or ``random''. 

\begin{table*}[htb]
    \centering
    \caption{One-shot hard voting for source classifiers in FedRF-TCA. Classification accuracy (\%) on Digit-Five dataset~\cite{ganin2016domain} with varying orders or asynchrony in $\Sigma y$ and $\W_{\RF}$ among clients. (B), (C), (D) and (E) represent that only one ordered or randomly selelcted source clients participants in the communication. (B) and (D) means that the message $\bSigma^{(i)} \bell^{(i)}, \W_{RF}^{(i)}$ are from the same source-target pair. (1) ``all'' means that all clients $i$ pass the message $\bSigma^{(i)} \bell^{(i)}$ or $\W_{RF}^{(i)}$; (2) ``order'' means message is passed one by one in order; (3) ``random'' means message is passed one by one randomly. For example, (C) signifies that in each round, $\bSigma^{(i)} \bell^{(i)}$ is sequentially transmitted to the target client from the ordered source clients $i$, while $\W_{RF}^{(j)}$ is sent to the server for aggregation from a randomly selected source client $j$.}
    \label{DigitFive_Asynchrony_Ablation}
    \begin{tabular}{ccccccc}
    \hline
    Setting ($\Sigma y \,$/$\, W_{\rm RF}$) & (A) all$\,$/$\,$all  & (B) ordered & (C) ordered$\,$/$\,$random & (D) random & (E) random$\,$/$\,$random \\ 
    \hline
    mm,sv,sy,up→mt & \underline{86.46} & 85.75 & 74.79 & 82.39 & \textbf{92.55} \\ 
    mt,sv,sy,up→mm & 53.71 & 45.11 & \textbf{57.55} & \underline{56.59} & 55.46 \\ 
    mt,mm,sv,sy→up & \underline{87.98} & 85.34 & 85.31 & \textbf{97.18} & 80.80 \\ 
    \hline
    \end{tabular}
\end{table*}

\Cref{DigitFive_Asynchrony_Ablation} demonstrates the strong robustness of the FedRF-TCA method regarding the order and asynchrony of message passing. Whether messages are transmitted in a specific order or randomly, and whether they are sent in pairs or not, FedRF-TCA consistently delivers high-performance results.

\end{document}